\documentclass[aoas]{imsart}
  \setattribute{journal}{name}{}
\usepackage{algorithm, algpseudocode}
\usepackage{amsthm,amsmath,amssymb,amsopn,natbib}
\RequirePackage[colorlinks,citecolor=blue,urlcolor=blue]{hyperref}
\usepackage{xcolor}
\usepackage{graphicx, epstopdf, subfig}

\startlocaldefs
\numberwithin{equation}{section}
\theoremstyle{plain}
\newtheorem{theorem}{Theorem}[section]

\newtheorem{corollary}[theorem]{Corollary}
\newtheorem{definition}[theorem]{Definition}
\newtheorem{lemma}[theorem]{Lemma}

\newtheorem{remark}[theorem]{Remark}

\newcommand{\BALD}{\begin{aligned}}
\newcommand{\EALD}{\end{aligned}}
\newcommand{\BALDS}{\begin{aligned*}}
\newcommand{\EALDS}{\end{aligned*}}
\newcommand{\BCAS}{\begin{cases}}
\newcommand{\ECAS}{\end{cases}}
\newcommand{\BEAS}{\begin{eqnarray*}}
\newcommand{\EEAS}{\end{eqnarray*}}
\newcommand{\BEQ}{\begin{equation}}
\newcommand{\EEQ}{\end{equation}}
\newcommand{\BIT}{\begin{itemize}}
\newcommand{\EIT}{\end{itemize}}
\newcommand{\BMAT}{\begin{bmatrix}}
\newcommand{\EMAT}{\end{bmatrix}}
\newcommand{\BNUM}{\begin{enumerate}}
\newcommand{\ENUM}{\end{enumerate}}
\newcommand{\eg}{e.g.}

\newcommand{\ie}{i.e.}
\newcommand{\cf}{cf.}

\newcommand{\BA}{\begin{array}}
\newcommand{\EA}{\end{array}}

\newcommand{\ones}{\mathbf 1}

\newcommand{\reals}{\mathbf{R}}

\DeclareMathOperator*{\argmin}{\arg\min}
\newcommand{\diag}{\mathop{\mathbf{diag}}}
\DeclareMathOperator{\Expect}{\mathbf{E}}

\DeclareMathOperator*{\minimize}{minimize}

\renewcommand{\Pr}{\mathbf{Pr}}

\DeclareMathOperator{\sign}{sign}

\newcommand{\pc}{\hspace{1pc}}

\newcommand{\abs}[1]{\left| #1 \right|}

\newcommand{\norm}[1]{\left\| #1 \right\|}

\newcommand{\iid}{\emph{i.i.d.}}

\newcommand{\htheta}{\hat{\theta}}

\newcommand{\cC}{\mathcal{C}}
\newcommand{\cA}{\mathcal{A}}

\newcommand{\cS}{\mathcal{S}}

\newcommand{\ta}{\tilde{a}}

\newcommand{\hTheta}{\hat{\Theta}}
\newcommand{\hSigma}{\hat{\Sigma}}

\newcommand{\hDelta}{\hat{\Delta}}

\newcommand{\btheta}{\bar{\theta}}
\newcommand{\betas}{\beta^*}
\newcommand{\hbeta}{\hat{\beta}}
\newcommand{\hbetad}{\hat{\beta}^d}
\DeclareMathOperator{\GC}{GC}
\newcommand{\cE}{\mathcal{E}}
\newcommand{\hE}{\hat{\cE}}
\DeclareMathOperator{\nz}{nz}
\DeclareMathOperator{\RE}{RE}
\newcommand{\bbeta}{\bar{\beta}}
\DeclareMathOperator{\HT}{HT}
\DeclareMathOperator{\ST}{ST}
\newcommand{\bbetaht}{\bbeta^{ht}}
\newcommand{\bbetast}{\bbeta^{st}}

\newcommand{\hs}{\hat{s}}
\newcommand{\ts}{\tilde{s}}

\newcommand{\tsigma}{\tilde{\sigma}}
\newcommand{\tbeta}{\tilde{\beta}}

\newcommand{\tbetaht}{\tbeta^{ht}}
\newcommand{\hC}{\hat{C}}
\newcommand{\hgamma}{\hat{\gamma}}

\newcommand{\hT}{\hat{T}}
\newcommand{\htau}{\hat{\tau}}

\newcommand{\drho}{\dot{\rho}}
\newcommand{\ddrho}{\ddot{\rho}}
\newcommand{\cR}{\mathcal{R}}

\endlocaldefs

\begin{document}

\begin{frontmatter}

\title{Communication-efficient sparse regression: a one-shot approach}
\runtitle{Communication-efficient sparse regression}

\begin{aug}
\author{\fnms{Jason D.} \snm{Lee}\thanksref{m1}\ead[label=e1]{jdl17@stanford.edu}},
\author{\fnms{Qiang} \snm{Liu}\thanksref{m2}\ead[label=e3]{qliu1@uci.edu}},
\author{\fnms{Yuekai} \snm{Sun}\thanksref{m1}\ead[label=e2]{yuekai@stanford.edu}},
\and
\author{\fnms{Jonathan E.} \snm{Taylor}\thanksref{m1}\ead[label=e4]{jonathan.taylor@stanford.edu}}

\runauthor{Lee et al.}

\affiliation{Stanford University\thanksmark{m1} and Dartmouth College\thanksmark{m2}}

\address{Institute for Computational and Mathematical Engineering\\
Stanford University \\
\printead{e1},
\printead*{e2}
}

\address{Department of Computer Science\\
Dartmouth College \\
\printead{e3}
}

\address{Department of Statistics\\
Stanford University \\
\printead{e4}
}
\end{aug}

\begin{abstract}
We devise a one-shot approach to distributed sparse regression in the high-dimensional setting. The key idea is to average``debiased'' or ``desparsified'' lasso estimators. We show the approach converges at the same rate as the lasso as long as the dataset is not split across too many machines. We also extend the approach to generalized linear models.
\end{abstract}



\end{frontmatter}


\section{Introduction}

Explosive growth in the size of modern datasets has fueled interest in distributed statistical learning. For examples, we refer to \cite{boyd2011distributed,dekel2012optimal,duchi2012dual,zhang2013communication} and the references therein. The problem arises, for example, when working with datasets that are too large to fit on a single machine and must be distributed across multiple machines. The main bottleneck in the distributed setting is usually communication between machines/processors, so the overarching goal of algorithm design is to minimize communication costs.

In distributed statistical learning, the simplest and most popular approach is \emph{averaging}: each machine forms a local estimator $\htheta_k$ with the portion of the data stored locally, and a ``master'' averages the local estimators to produce an aggregate estimator: $\btheta = \frac1m \sum_{k=1}^m \htheta_k.$ Averaging was first studied by \cite{mcdonald2009efficient} for multinomial regression. They derive non-asymptotic error bounds on the estimation error that show averaging reduces the variance of the local estimators, but has no effect on the bias (from the centralized solution). In follow-up work, \cite{zinkevich2010parallelized} studied a variant of averaging where each machine computes a local estimator with stochastic gradient descent (SGD) on a random subset of the dataset. They show, among other things, that their estimator converges to the centralized estimator.

More recently, \cite{zhang2013communication} studied averaged empirical risk minimization (ERM). They show that the mean squared error (MSE) of the averaged ERM decays like $O\bigl(N^{-\frac12} + \frac{m}{N}\bigr),$ where $m$ is the number of machines and $N$ is the total number of samples. Thus, so long as $m \lesssim \sqrt{N},$ the averaged ERM matches the $N^{-\frac12}$ convergence rate of the centralized ERM.
Even more recently, \cite{rosenblatt2014optimality} studied the optimality of averaged ERM in two asymptotic settings: $N\to\infty$, $m,p$ fixed and $p,n\to\infty$, $\frac{p}{n}\to \mu_l\in(0,1)$, where $n=\frac{N}{m}$ is the number of samples per machine. They show that in the $n\to\infty$, $p$ fixed setting, the averaged ERM is first-order equivalent to the centralized ERM. However, when $p,n\to\infty,$ the averaged ERM is suboptimal (versus the centralized ERM).

We develop an approach to distributed statistical learning in the high-di\-mensional setting. Since $p\gtrsim n,$ regularization is essential. At a high level, the key idea is to average local \emph{debiased} regularized M-estimators. We show that our averaged estimator converges at the same rate as the centralized regularized M-estimator.

\section{Background on the lasso and the debiased lasso}

To keep things simple, we focus on sparse linear regression. Consider the sparse linear model
\[
y = X\betas + \epsilon,
\]
where the rows of $X\in\reals^{n\times p}$ are predictors, and the components of $y\in\reals^n$ are the responses. To keep things simple, we assume
\BNUM
\item[(A1)] the predictors $x\in\reals^p$ are independent subgaussian random vectors whose covariance $\Sigma$ has smallest has smallest eigenvalue $\lambda_{\min}(\Sigma)$;
\item[(A2)] the regression coefficients $\betas\in\reals^p$ are $s$-sparse, \ie{} all but $s$ components of $\betas$ are zero;
\item[(A3)] the components of $\epsilon\in\reals^{n}$ are independent, mean zero subgaussian random variables.
\ENUM

Given the predictors and responses, the lasso estimates $\betas$ by
\[
\hbeta := \argmin_{\beta\in\reals^p}\frac{1}{2n}\|y - X\beta\|_2^2 + \lambda\|\beta\|_1.
\]
There is a well-developed theory of the lasso that says, under suitable assumptions on $X,$ the lasso estimator $\hbeta$ is nearly as close to $\betas$ as the \emph{oracle estimator}: $X_{\nz(\betas)}^\dagger y$ (\eg{} see \cite{hastie2015statistical}, Chapter 11 for an overview).
More precisely, under some conditions on $\frac1nX^TX,$ the MSE of the lasso estimator is roughly $\frac{s\log p}{n}.$ Since the MSE of the oracle estimator is (roughly) $\frac{s}{n},$ the lasso estimator is almost as good as the oracle estimator.

However, the lasso estimator is also biased\footnote{We refer to Section 2.2 in \cite{javanmard2013confidence} for a more formal discussion of the bias of the lasso estimator.}. Since averaging only reduces variance, not bias, we gain (almost) nothing by averaging the biased lasso estimators. That is, it is possible to show if we naively averaged local lasso estimators, the MSE of the averaged estimator is of the same order as that of the local estimators. The key to overcoming the bias of the averaged lasso estimator is to ``debias'' the lasso estimators before averaging.

The \emph{debiased lasso estimator} by \cite{javanmard2013confidence} is
\BEQ
\hbetad := \hbeta + \frac1n\hTheta X^T(y - X\hbeta),
\label{eq:debiased-lasso}
\EEQ
where $\hbeta$ is the lasso estimator and $\hTheta\in\reals^{p\times p}$ is an approximate inverse to $\hSigma = \frac1nX^TX.$ Intuitively, the debiased lasso estimator trades bias for variance. The trade-off is obvious when $\hSigma$ is non-singular: setting $\hTheta = \hSigma^{-1}$ gives the ordinary least squares (OLS) estimator $(X^TX)^{-1} X^T y.$

Another way to interpret the debiased lasso estimator is a corrected estimator that compensates for the bias incurred by shrinkage. By the optimality conditions of the lasso, the correction term $\frac1nX^T(y - X\hbeta)$ is a subgradient of $\lambda\norm{\cdot}_1$ at $\hbeta.$ By adding a term proportional to the subgradient of the regularizer, the debiased lasso estimator compensates for the bias incurred by regularization. The debiased lasso estimator has previously been used to perform inference on the regression coefficients in high-dimensional regression models. We refer to the papers by \cite{javanmard2013confidence,van2013asymptotically,zhang2014confidence,belloni2011inference} for details.

The choice of $\hTheta$ in the correction term is crucial to the performance of the debiased estimator. \cite{javanmard2013confidence} suggest forming $\hTheta$ row by row: the $j$-th row of $\hTheta$ is the optimum of
\BEQ
\begin{aligned}
&\minimize_{\theta\in\reals^p} & & \theta^T\hSigma \theta \\
&\text{subject to} & & \|\hSigma \theta - e_j\|_\infty \le \delta.
\end{aligned}
\label{eq:javanmard-LDPE}
\EEQ
The parameter $\delta$ should large enough to keep the problem feasible, but as small as possible to keep the bias (of the debiased lasso estimator) small. As we shall see, when the rows of $X$ are subgaussian, setting $\delta \sim\bigl(\frac{\log p}{n}\bigr)^{\frac12}$ is usually large enough to keep \eqref{eq:javanmard-LDPE} feasible.

\begin{definition}[Generalized coherence]
\label{def:generalized-coherence}
Given $X\in\reals^{n\times p},$ let $\hSigma = \frac1nX^TX.$ The \emph{generalized coherence} between $\hSigma$ and $\Theta\in\reals^{p\times p}$ is
\[\textstyle
\GC(\hSigma,\Theta) = \max_{j\in[p]}\|\hSigma \Theta_j^T - e_j\|_\infty.
\]
\end{definition}

\begin{lemma}[\cite{javanmard2013confidence}]
\label{lem:GC}
Under (A1), when $16\kappa\sigma_x^4 n > \log p,$ the event
\[
\cE_{\GC}(\hSigma) := \Bigl\{\GC(\hSigma,\Sigma^{-1}) \le \frac{8}{\sqrt{c_1}}\sqrt{\kappa}\sigma_x^2\Bigl(\frac{\log p}{n}\Bigr)^{\frac12}\Bigr\}
\]
occurs with probability at least $1 - 2p^{-2}$ for some $c_1 > 0,$ where $\kappa := \frac{\lambda_{\max}(\Sigma)}{\lambda_{\min}(\Sigma)}$ is the con\-dition number of $\Sigma$.
\end{lemma}

As we shall see, the bias of the debiased lasso estimate is of higher order than its variance under suitable conditions on $\hSigma.$ In particular, we require $\hSigma$ to satisfy the \emph{restricted eigenvalue (RE) condition}.

\begin{definition}[RE condition]
\label{def:RE}
For any $\cS\subset[p],$ let
\[
\cC(\cS) := \{\Delta\in\reals^p\mid \norm{\Delta_{\cS^c}}_1 \le 3\norm{\Delta{x_{\cS}}}_1\}.
\]
We say $\hSigma$ satisfies the RE condition on the cone $\cC(\cS)$ when
\[
\Delta^T\hSigma\Delta \ge \mu_l\|\Delta_{\cS}\|_2^2
\]
for some $\mu_l > 0$ and any $\Delta\in\cC(\cS).$
\end{definition}

The RE condition requires $\hSigma$ to be positive definite on $\cC(\cS).$ When the rows of $X\in\reals^{n\times p}$ are \iid{} Gaussian random vectors, \cite{raskutti2010restricted} show there are constants $\mu_1,\mu_2 > 0$ such that
\[
\frac1n\norm{X\Delta}_2^2 \ge \mu_1\norm{\Delta}_2^2 - \mu_2\frac{\log p}{n}\norm{\Delta}_1^2\text{ for any }\Delta\in\reals^p
\]
with probability at least $1 - c_2\exp\left(-c_2n\right).$ Their result implies the RE condition holds on $\cC(\cS)$ (for any $\cS\subset[p]$) as long as $n \gtrsim \abs{S}\log p,$ even when there are dependencies among the predictors. Their result was extended to subgaussian designs by \cite{rudelson2013reconstruction}, also allowing for dependencies among the covariates. We summarize their result in a lemma.

\begin{lemma}
\label{lem:rudelson-zhou-RE}
Under (A1), when $n > 4000 \ts \sigma_x^2 \log\bigl(\frac{60\sqrt{2}ep}{\ts}\bigr)$ and $p > \ts$, where $\ts := s+ 25920\kappa s$, the event
\[
\cE_{\RE}(X) = \Bigl\{\Delta^T\hSigma\Delta \ge \frac12\lambda_{\min}(\Sigma)\|\Delta_S\|_2^2\text{ for any }\Delta\in \cC(S)\Bigr\}
\]
occurs with probability at least $1 - 2e^{-\frac{n}{4000\sigma_x^4}}.$
\end{lemma}

\begin{proof}
The lemma is a consequence of \cite{rudelson2013reconstruction}, Theorem 6. In their notation, we set $\delta = \frac{1}{\sqrt{2}},$ $k_0 = 3$ and bound $\max_{j\in[p]}\norm{Ae_j}_2^2$ and $K(s_0, k_0, \Sigma^{\frac12})$ by $\lambda_{\max}(\Sigma)$ and $\lambda_{\min}(\Sigma)^{-\frac12}.$
\end{proof}

When the RE condition holds, the lasso and debiased lasso estimators are consistent for a suitable choice of the regularization parameter $\lambda.$ The parameter $\lambda$ should be large enough to dominate the ``empirical process'' part of the problem: $\frac1n\norm{X^Ty}_\infty,$ but as small as possible to reduce the bias incurred by regularization. As we shall see, setting $\lambda \sim \sigma_y\bigl(\frac{\log p}{n}\bigr)^{\frac12}$ is a good choice.

\begin{lemma}
\label{lem:lambda-large}
Under (A3),
\[
\frac1n\|X^T\epsilon\|_\infty \le {\textstyle\max_{j\in[p]}(\hSigma_{j,j})^{\frac12}}\sigma_y\Bigl(\frac{3\log p}{c_2n}\Bigr)^{\frac12}
\]
with probability at least $1-ep^{-2}$ for any (non-random) $X\in\reals^{n\times p}$.
\end{lemma}

When $\hSigma$ satisfies the RE condition and $\lambda$ is large enough, the lasso and debiased lasso estimators are consistent.

\begin{lemma}[\cite{negahban2012unified}]
\label{lem:lasso-consistency}
Under (A2) and (A3), suppose $\hSigma$ satisfies the RE condition on $\cC^*$ with constant $\mu_l$ and $\frac1n\|X^T\epsilon\|_\infty \le \lambda$,
\[
\|\hbeta - \beta\|_1 \le \frac{3}{\mu_l} s\lambda\text{ and }\|\hbeta - \beta\|_2 \le \frac{3}{\mu_l}\sqrt{s}\lambda.
\]
\end{lemma}

When the lasso estimator is consistent, the debiased lasso estimator is also consistent. Further, it is possible to show that the bias of the debiased estimator is of higher order than its variance. Similar results by \cite{javanmard2013confidence,van2013asymptotically,zhang2014confidence,belloni2011inference} are the key step in showing the asymptotic normality of the (components of) the debiased lasso estimator. The result we state is essentially \cite{javanmard2013confidence}, Theorem 2.3.

\begin{lemma}
\label{lem:debiased-lasso-consistency}
Under the conditions of Lemma \ref{lem:lasso-consistency}, when $(\hSigma,\hTheta)$ has generalized incoherence $\delta,$ the debiased lasso estimator has the form
\[
\hbetad = \betas + \frac1n\hTheta X^T\epsilon + \hDelta,
\]
where $\|\hDelta\|_\infty \le \frac{3\delta}{\mu_l}s\lambda.$
\end{lemma}

Lemma \ref{lem:debiased-lasso-consistency}, together with Lemmas \ref{lem:lambda-large} and \ref{lem:GC}, shows that the bias of the debiased lasso estimator is of higher order than its variance. In particular, setting $\lambda$ and $\delta$ according to Lemmas \ref{lem:lambda-large} and \ref{lem:GC} gives a bias term $\|\hDelta\|_\infty$ that is $O\bigl(\frac{s\log p}{n}\bigr).$ By comparison, the variance term $\frac1n\|\hTheta X^T\epsilon\|_\infty$ is the maximum of $p$ subgaussian random variables with mean zero and variances of $O(1),$ which is $O\bigl(\bigl(\frac{\log p}{n}\bigr)^{\frac12}\bigr).$
Thus the bias term is of higher order than the variance term as long as $n \gtrsim s^2 \log p$.

\begin{corollary}
\label{cor:debiased-lasso-consistency}
Under (A2), (A3), and the conditions of Lemma \ref{lem:lasso-consistency}, when $(\hSigma,\hTheta)$ has generalized incoherence $\delta'\bigl(\frac{\log p}{n}\bigr)^{\frac12}$ and we set $\lambda = \max_{j\in[p]}(\hSigma_{j,j})^{\frac12}\sigma_y\bigl(\frac{3\log p}{c_2n}\bigr)^{\frac12},$
\[
\|\hDelta\|_\infty \le \frac{3\sqrt{3}}{\sqrt{c_2}}\frac{\delta'\max_{j\in[p]}(\hSigma_{j,j})^{\frac12}}{\mu_l}\sigma_y\frac{s\log p}{n}.
\]
\end{corollary}

\section{Averaging debiased lassos}

Recall the problem setup: we are given $N$ samples of the form $(x_i,y_i)$ distributed across $m$ machines:
\[
X = \BMAT
X_1 \\
\vdots \\
X_m
\EMAT
,\quad
y = \BMAT
y_1 \\
\vdots \\
y_m
\EMAT.
\]
The $k$-th machine has local predictors $X_k\in\reals^{n_k\times p}$ and responses $y_k\in\reals^{n_k}.$ To keep things simple, we assume the data is evenly distributed, \ie{} $n_1 = \dots = n_k = n = \frac{N}{m}.$ The \emph{averaged debiased lasso} estimator (for lack of a better name) is
\BEQ
\bbeta = \frac1m\sum_{ k= 1}^m \hbetad_k = \frac1m\sum_{k = 1}^m \hbeta_k + \hTheta_k X_k^T(y_k - X_k\hbeta_k),
\label{eq:averaged-debiased-lasso}
\EEQ
We study the error of the averaged debiased lasso in the $\ell_\infty$ norm.

\begin{lemma}
\label{lem:averaged-debiased-lasso-consistency}
Suppose the local sparse regression problem on each machine satisfies the conditions of Corollary \ref{cor:debiased-lasso-consistency}, that is when $m \le p$,
\BNUM
\item $\{\hSigma_k\}_{k\in[m]}$ satisfy the RE condition on $\cC^*$ with constant $\mu_l,$
\item $\{(\hSigma_k,\hTheta_k)\}_{k\in[m]}$ have generalized incoherence $c_{\GC}\bigl(\frac{\log p}{n}\bigr)^{\frac12},$
\item we set $\lambda_1 = \dots = \lambda_m = c_{\Sigma}\sigma_y\bigr(\frac{3\log p}{c_2n}\bigr)^{\frac12}.$
\ENUM
Then
\[
\|\bbeta - \betas\|_\infty \le c\sigma_y\Bigl(\Bigl(\frac{c_{\Omega}\log p }{N}\Bigr)^{\frac12}+ \frac{c_{\GC}c_{\Sigma}}{\mu_l}\sigma_y\frac{s\log p}{n}\Bigr)
\]
with probability at least $1 - ep^{-1},$ where $c > 0$ is a universal constant, $c_{\Omega} := \max_{j\in[p],\,k\in[m]}((\hTheta_k \hSigma_k \hTheta_k^T)_{j,j})$ and $c_{\Sigma} := \max_{j\in[p],k\in[m]}((\hSigma_k)_{j,j})^{\frac12}.$
\end{lemma}

Lemma \ref{lem:averaged-debiased-lasso-consistency} hints at the performance of the averaged debiased lasso. In particular, we note the first term is $O\bigl(\bigl(\frac{\log p}{N}\bigr)^{\frac12}\bigr),$ which matches the convergence rate of the centralized estimator. When $n$ is large enough, $\frac{s\log p}{n}$ is negligible compared to $\bigl(\frac{\log p}{N}\bigr)^{\frac12},$ and the error is $O\bigl(\bigl(\frac{\log p}{N}\bigr)^{\frac12}\,\bigr).$

Finally, we show the conditions of Lemma \ref{lem:averaged-debiased-lasso-consistency} occur with high probability when the rows of $X$ are independent subgaussian random vectors.

\begin{theorem}
\label{thm:averaged-debiased-lasso-consistency}
Under (A1), (A2), and (A3), when $m < p$, $p > \ts$,
\BNUM
\item $n > \max\bigl\{ 4000 \ts \sigma_x^2 \log(\frac{60\sqrt{2}ep}{\ts}),\,8000 \sigma_x^4 \log p,\,\frac{3}{c_1}\max\{\sigma_x^2,\,\sigma_x\}\log p\bigr\}$,
\item we set $\lambda_1 = \dots = \lambda_m = \max_{j\in[p],k\in[m]}((\hSigma_k\bigr)_{j,j})^{\frac12}\sigma_y\bigr(\frac{3\log p}{c_2n}\bigr)^{\frac12}$,
\item we set $\delta_1 = \dots = \delta_m = \frac{8}{\sqrt{c_1}}\sqrt{\kappa}\sigma_x^2\bigl(\frac{\log p}{n}\bigr)^{\frac12}$ and form $\{\hTheta_k\}_{k\in[m]}$ by \eqref{eq:javanmard-LDPE},
\ENUM
\[
\|\bbeta - \betas\|_\infty \le c\Bigl(\sigma_y\biggl(\frac{\max_{j\in[p]}\Sigma_{j,j}^{-1}\log p }{N}\biggr)^{\frac12}+ \frac{\sqrt{\kappa}\max_{j\in[p]}(\Sigma_{j,j})^{\frac12}}{\lambda_{\min}(\Sigma)}\sigma_x^2\sigma_y\frac{s\log p}{n}\Bigr)
\]
with probability at least $1 - (8 + e)p^{-1}$ for some universal constant $c > 0.$
\end{theorem}

\begin{proof}
We start with the conclusion of Lemma \ref{lem:averaged-debiased-lasso-consistency}:
\[
\|\bbeta - \betas\|_\infty \le \sigma_y\Bigl(\frac{2c_{\Omega}\log p }{c_2N}\Bigr)^{\frac12}+ \frac{3\sqrt{3}}{\sqrt{c_2}}\frac{c_{\GC}c_{\Sigma}}{\mu_l}\sigma_y\frac{s\log p}{n}.
\]
First, we show that the two constants $c_{\Omega} = \max_{j\in[p],\,k\in[m]}(\hTheta_k \hSigma_k \hTheta_k^T)_{j,j}$ and $c_{\Sigma}:= \max_{j\in[p],k\in[m]}((\hSigma_k)_{j,j})^{\frac12}$ are bounded with high probability.


\begin{lemma}
\label{lem:c-var-bounded}
Under (A1),
\[\textstyle
\Pr\bigl(\max_{j\in[p]}\Sigma_j^{-1}\hSigma\Sigma_j^{-1} > 2\max_{j\in[p]}\Sigma_{j,j}^{-1}\bigr) \le 2pe^{-c_1\min\{\frac{n}{\sigma_x^2},\frac{n}{\sigma_x}\}}
\]
for some universal constant $c_1 > 0.$
\end{lemma}

Since we form $\{\hTheta_k\}_{k\in[m]}$ by \eqref{eq:javanmard-LDPE},
\[
(\hTheta_k \hSigma_k \hTheta_k^T)_{j,j} \le {\textstyle\max_{j\in[p]}(\Sigma^{-1}\hSigma_k \Sigma^{-1}))_{j,j}}.
\]
Lemma \ref{lem:c-var-bounded} implies
\[\textstyle
\max_{j\in[p]}(\Sigma^{-1}\hSigma_k \Sigma^{-1}))_{j,j} \le 2\max_{j\in[p]}\Sigma_{j,j}^{-1}\text{ for each }k\in[m]
\]
with probability at least $1 - 2pe^{-c_1\min\{\frac{n}{\sigma_x^2},\frac{n}{\sigma_x}\}}.$

\begin{lemma}
\label{lem:c-bias-bounded}
Under (A1),
\[\textstyle
\Pr(\max_{j\in[p]}(\hSigma_{j,j})^{\frac12} > \sqrt{2}\max_{j\in[p]}(\Sigma_{j,j})^{\frac12}) \le 2pe^{-c_1\min\{\frac{n}{16\sigma_x^2},\frac{n}{4\sigma_x}\}}
\]
for some universal constant $c_1 > 0.$
\end{lemma}

We put the pieces together to obtain the stated result:
\BNUM
\item By Lemma \ref{lem:c-var-bounded} (and a union bound over $k\in[m]$),
\[\textstyle
\Pr(c_{\Omega} \ge 2\max_j\Sigma_{j,j}^{-1}) \le 2mpe^{-c_1\min\{\frac{n}{\sigma_x^2},\frac{n}{\sigma_x}\}}.
\]
Since $m \le p,$ when $n > \frac{3}{c_1}\max\{\sigma_x^2,\sigma_x\}\log p,$
\[
\Pr\bigl(c_{\Omega} < 2\max_j\Sigma_{j,j}^{-1}\bigr) \ge 1-2p^{-1}.
\]
\item By Lemma \ref{lem:c-bias-bounded} (and a union bound over $k\in[m]$),
\[\textstyle
\Pr(c_{\Sigma} < \sqrt{2}\max_{j\in[p]}(\Sigma_{j,j})^{\frac12}) \ge 1-2mpe^{-c_1\min\{\frac{n}{16\sigma_x^2},\frac{n}{4\sigma_x}\}}.
\]
When $n > \frac{3}{c_1}\max\{\sigma_x^2,\sigma_x\}\log p,$ the right side is again at most $2p^{-1}.$

\item By Lemma \ref{lem:rudelson-zhou-RE}, as long as
\[\textstyle
n > \max\{4000 \ts \sigma_x^2 \log(\frac{60\sqrt{2}ep}{\ts}),\,8000 \sigma_x^4 \log p\},
\]
$\hSigma_1,\dots,\hSigma_m$ all satisfy the RE condition with probability at least
\[
1 - 2me^{-\frac{n}{4000\sigma_x^4}} \ge 1 - 2p^{-1}.
\]
\item By Lemma \ref{lem:GC},
\[
\Pr\bigl(\cap_{k\in[m]}\cE_{\GC}(\hSigma_k)\bigr) \ge 1 - 2p^{-2}.
\]
Since $m < p,$ the probability is at least $1 - 2p^{-1}.$
\ENUM
We apply the bounds $c_{\Omega} \le 2\max_{j\in[p]}\Sigma_{j,j}^{-1}$, $c_{\Sigma} \le \sqrt{2}\max_{j\in[p]}(\Sigma_{j,j})^{\frac12}$, $c_{\GC} = \frac{8}{\sqrt{c_1}}\sqrt{\kappa}\sigma_x^2,$ and $\mu_l = \frac12\lambda_{\min}(\Sigma)$ to obtain
\[
\|\bbeta - \betas\|_\infty \le \sigma_y\biggl(\frac{4\max_{j\in[p]}\Sigma_{j,j}^{-1}\log p }{c_2N}\biggr)^{\frac12}+ \frac{48\sqrt{6}}{\sqrt{c_1c_2}}\frac{\sqrt{\kappa}\max_{j\in[p]}(\Sigma_{j,j})^{\frac12}}{\lambda_{\min}(\Sigma)}\sigma_x^2\sigma_y\frac{s\log p}{n}.
\]
\end{proof}

We validate our theoretical results with simulations. First, we study the estimation error of the averaged debiased lasso in $\ell_\infty$ norm. To focus on the effect of averaging, we grow the number of machines $m$ linearly with the (total) sample size $N.$ In other words, we fix the sample size per machine $n$ and grow the total sample size $N$ by adding machines. Figure \ref{fig:n-fixed-linf-error} compares the estimation error (in $\ell_\infty$ norm) of the averaged debiased lasso estimator with that of the centralized lasso. We see the estimation error of the averaged debiased lasso estimator is comparable to that of the centralized lasso, while that of the naive averaged lasso is much worse.


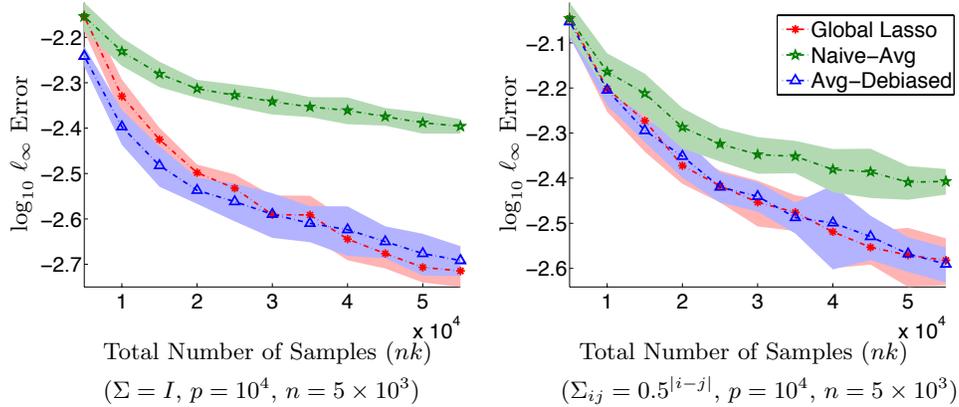
\begin{figure}
  \begin{tabular}{cc}
    \includegraphics[width=.45\textwidth]{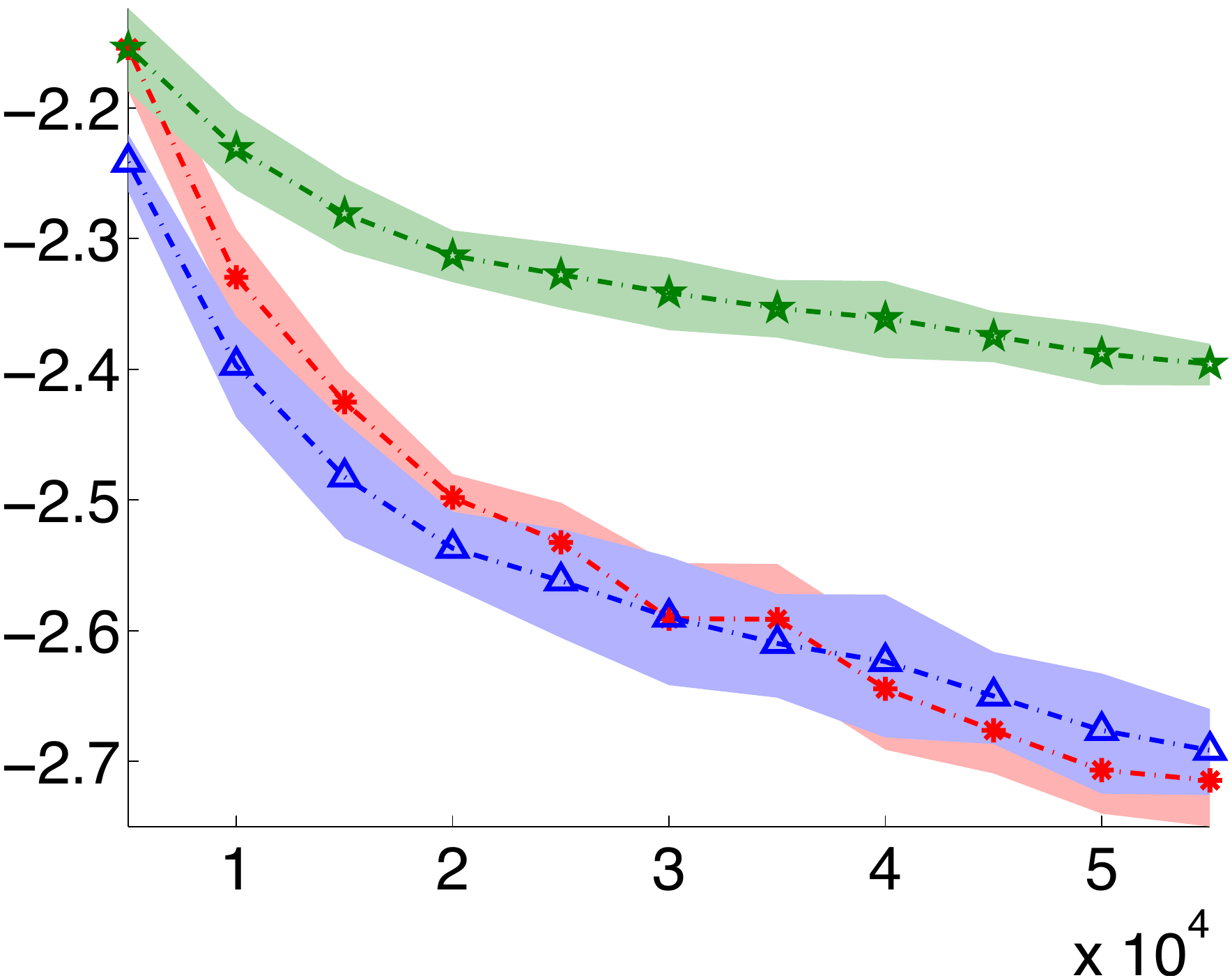}
    ~~~~~
    \includegraphics[width=.45\textwidth]{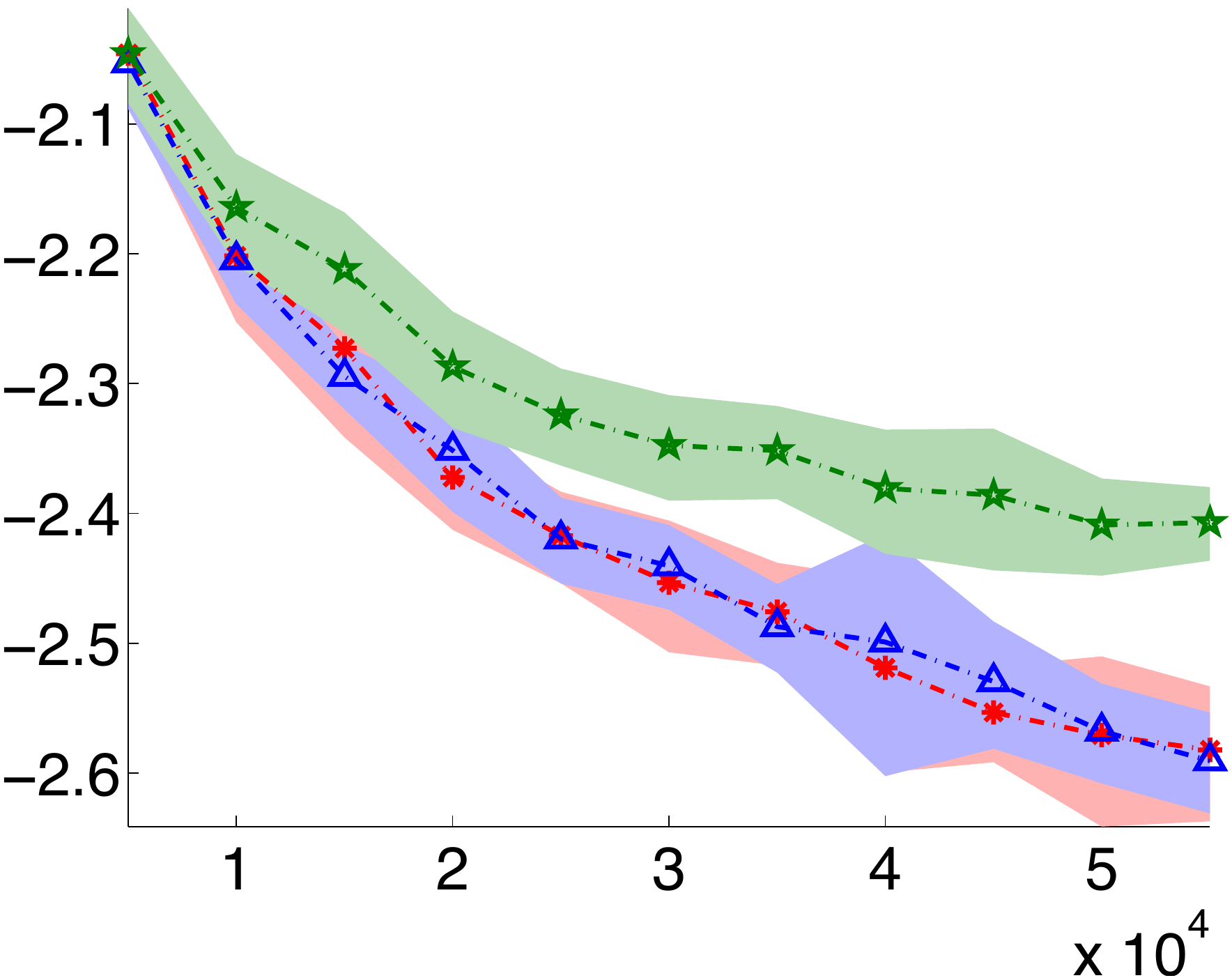}
     \hspace{-8em} \raisebox{10em}{\includegraphics[width=.2\textwidth]{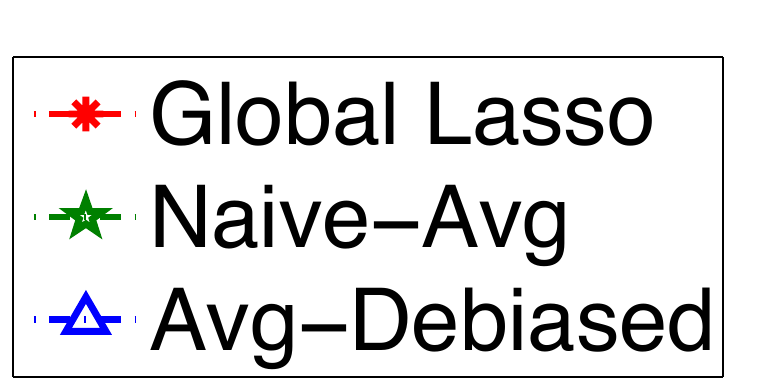}}  \\
   \end{tabular}
   \begin{picture}(0,18)(0,0)
   \put(-150,15){$\text{Total Number of Samples ($nk$)}$}
   \put(30,15){$\text{Total Number of Samples ($nk$)}$}
   \put(-150,0){($\Sigma = I$, $p =10^4$, $n= 5\times10^3$)}
   \put(20,0){ ($\Sigma_{ij} = 0.5^{\abs{i-j}}$, $p =10^4$, $n= 5\times10^3$)}
   \put(-185, 60){\rotatebox{90}{ $\log_{10}$ $\ell_{\infty}$ Error}}
   \put(0, 60){\rotatebox{90}{ $\log_{10}$ $\ell_{\infty}$ Error}}
   \end{picture}
  \caption{The estimation error (in $\ell_\infty$ norm) of the averaged debiased lasso estimator versus that of the centralized lasso when the predictors are Gaussian. In both settings, the estimation error of the averaged debiased estimator is comparable to that of the centralized lasso, while that of the naive averaged lasso is much worse.}
  \label{fig:n-fixed-linf-error}
\end{figure}

We conduct a second set of simulations to study the effect of the number of machines on the estimation effort of the averaged estimator. To focus on the effect of the number of machines $k,$ we fix the (total) sample size $N$ and vary the number of machines the samples are distributed across. Figure \ref{fig:N-fixed-linf-error} shows how the estimation error (in $\ell_\infty$ norm) of the averaged estimator grows as the number of machines grows. When the number of machines is small, the estimation error of the averaged estimator is comparable to that of the centralized lasso. However, when the number of machines exceeds a certain threshold, the estimation error grows with the number of machines. This is consistent with the prediction of Theorem \ref{thm:averaged-debiased-lasso-consistency}: when the number of machines exceeds a certain threshold, the bias term of order $\frac{s \log p}{n}$ becomes dominant.


\begin{figure}
  \begin{tabular}{cc}
    \includegraphics[width=.45\textwidth]{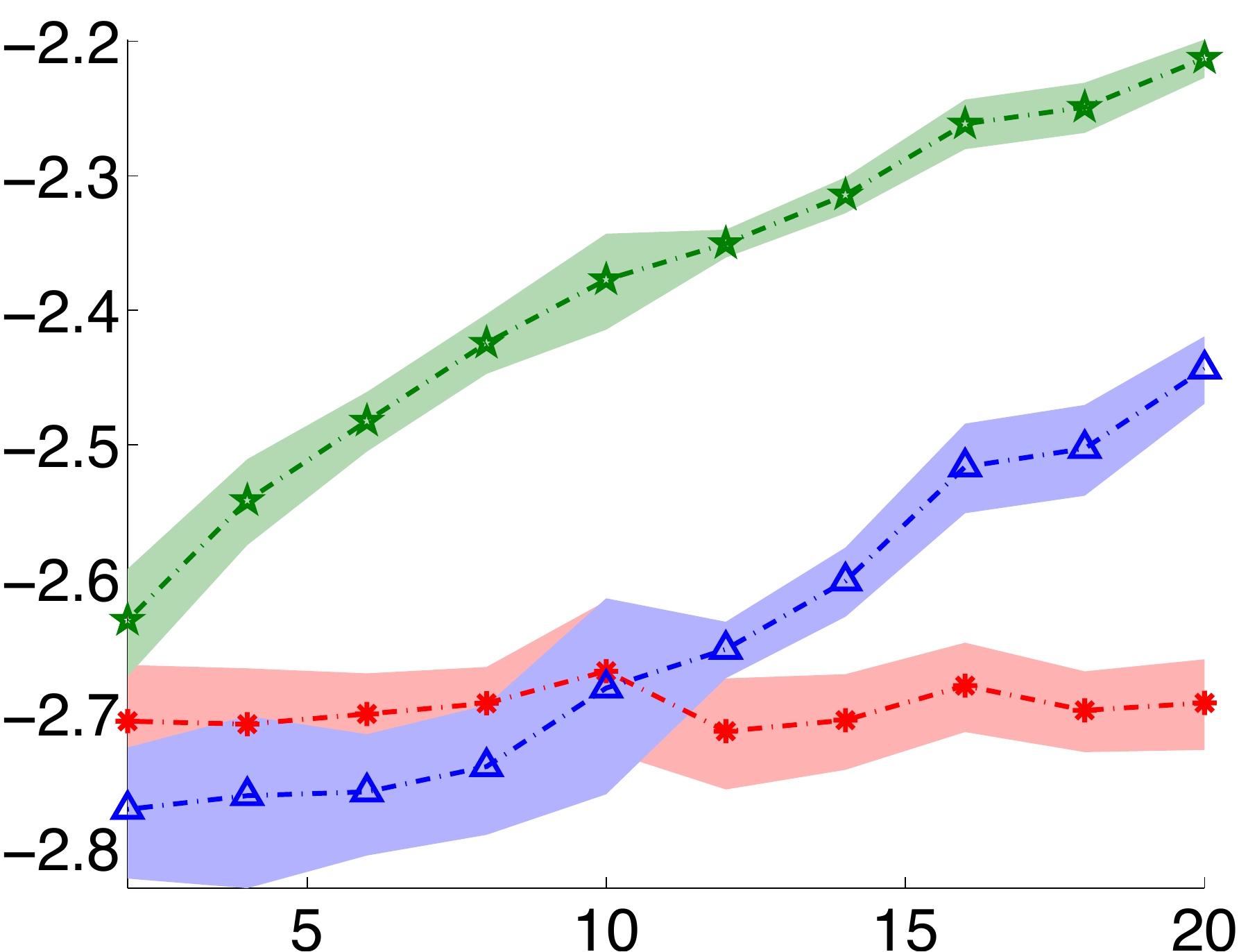}
    ~~~~~
    \includegraphics[width=.45\textwidth]{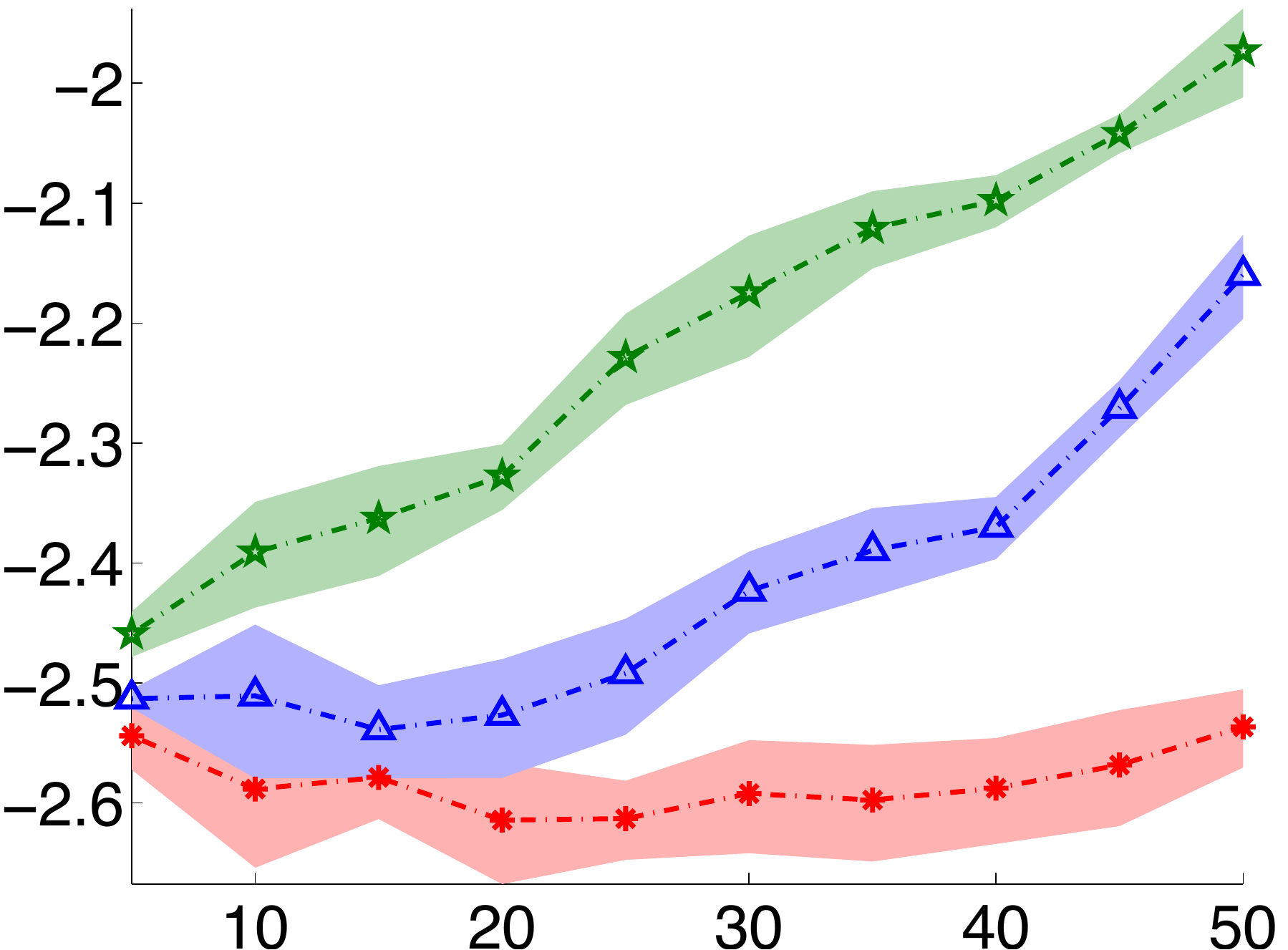}
     \hspace{-15em} \raisebox{10em}{\includegraphics[width=.2\textwidth]{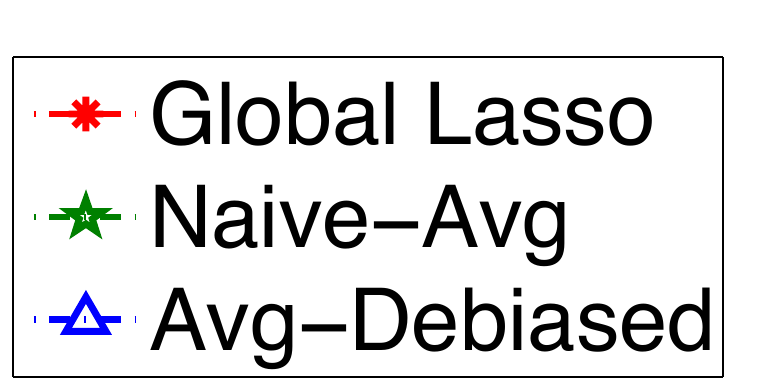}} \hspace{7em} \\
   \end{tabular}
   \begin{picture}(0,18)(0,0)
   \put(-150,15){$\text{Number of Machines ($k$)}$}
   \put(30,15){$\text{Number of Machines ($k$)}$}
   \put(-150,0){($\Sigma = I$, $p =10^4$, $nk= 2\times10^5$)}
   \put(20,0){ ($\Sigma_{ij} = 0.5^{\abs{i-j}}$, $p =10^4$, $nk= 2\times10^5$)}
   \put(-185, 60){\rotatebox{90}{ $\log_{10}$ $\ell_{\infty}$ Error}}
   \put(0, 60){\rotatebox{90}{ $\log_{10}$ $\ell_{\infty}$ Error}}
   \end{picture}
  \caption{
  The estimation error (in $\ell_\infty$ norm) of the averaged estimator as the number of machines $k$ vary. When the number of machines is small, the error is comparable to that of the centralized lasso. However, when the number of machines exceeds a certain threshold, the bias term (which grows linearly in $k$) is dominant, and the performance of the averaged estimator degrades.}
  \label{fig:N-fixed-linf-error}
\end{figure}
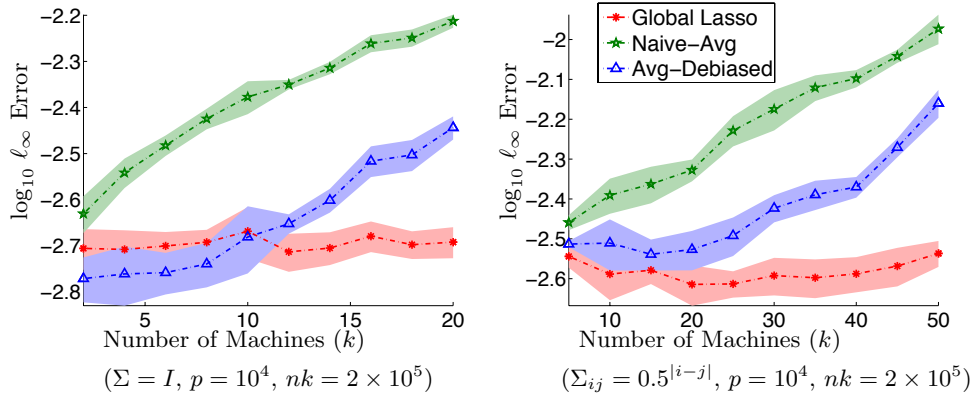

The averaged debiased lasso has one serious drawback versus the lasso: $\bbeta$ is usually dense. The density of $\bbeta$ detracts from the intrepretability of the coefficients and makes the estimation error large in the $\ell_2$ and $\ell_1$ norms. To remedy both problems, we threshold the averaged debiased lasso:
\begin{align*}
\HT_t(\bbeta) &\gets \bbeta_j \cdot \ones_{\{\abs{\bbeta_j} \ge t\}},\\
\ST_t(\bbeta) &\gets \sign(\bbeta_j)\cdot  \max\{|\bbeta_j|-t,\,0\}.
\end{align*}

As we shall see, both hard and soft-thresholding give sparse aggregates that are close to $\betas$ in $\ell_2$ norm.

\begin{lemma}
As long as $t > \|\bbeta - \betas\|_\infty,$ $\bbetaht := \HT_t(\bbeta)$ satisfies
\BNUM
\item $\|\bbetaht- \betas\|_\infty \le 2t,$
\item $\|\bbetaht - \betas\|_2 \le 2\sqrt{2s}t,$
\item $\|\bbetaht - \betas\|_1 \le 2\sqrt{2}st.$
\ENUM
The analogous result also holds for $\bbetast := \ST_t(\bbeta).$
\label{lem:thresh-consistency}
\end{lemma}

\begin{proof}
By the triangle inequality,
\begin{align*}
\|\bbetaht - \betas\|_\infty &\le \|\bbetaht - \bbeta\|_\infty + \|\bbeta- \betas\|_\infty \\
&\le t+\norm{\bbeta - \betas}_\infty\\
&\le 2t.
\end{align*}
Since $t > \norm{\bbeta - \betas }_\infty,$ $\bbetaht_j = 0$ whenever $\betas_j = 0.$ Thus $\bbetaht$ is $s$-sparse and $\bbetaht -\betas$ is $2s$-sparse. By the equivalence between the $\ell_\infty$ and $\ell_2$, $\ell_1$ norms,
\begin{align*}
\|\bbetaht - \betas\|_2 \le 2\sqrt{2s} t,\\
\|\bbetaht - \betas\|_1 \le 2\sqrt{2}s t.
\end{align*}
The argument for $\bbetast$ is similar.
\end{proof}

By combining Lemma \ref{lem:thresh-consistency} with Theorem \ref{thm:averaged-debiased-lasso-consistency}, we show that $\bbetaht$ converges at the same rates as the centralized lasso.

\begin{theorem}
\label{thm:bbetaht-consistency}
Under the conditions of Theorem \ref{thm:averaged-debiased-lasso-consistency}, hard-thresholding $\bbeta$ at $\sigma_y\Bigl(\frac{4\max_{j\in[p]}\Sigma_{j,j}^{-1}\log p }{c_2N}\Bigr)^{\frac12}+ \frac{48\sqrt{6}}{\sqrt{c_1c_2}}\frac{\sqrt{\kappa}\max_{j\in[p]}(\Sigma_{j,j})^{\frac12}}{\lambda_{\min}(\Sigma)}\sigma_x^2\sigma_y\frac{s\log p}{n}$ gives
\BNUM
\item $\|\bbetaht- \betas\|_\infty \lesssim_P \sigma_y\Bigl(\frac{\max_{j\in[p]}\Sigma_{j,j}^{-1}\log p }{N}\Bigr)^{\frac12}+ \frac{\sqrt{\kappa}\max_{j\in[p]}(\Sigma_{j,j})^{\frac12}}{\lambda_{\min}(\Sigma)}\sigma_x^2\sigma_y\frac{s\log p}{n},$
\item $\|\bbetaht - \betas\|_2 \lesssim_P \sigma_y\Bigl(\frac{\max_{j\in[p]}\Sigma_{j,j}^{-1}s\log p }{N}\Bigr)^{\frac12}+ \frac{\sqrt{\kappa}\max_{j\in[p]}(\Sigma_{j,j})^{\frac12}}{\lambda_{\min}(\Sigma)}\sigma_x^2\sigma_y\frac{s^{\frac32}\log p}{n},$
\item $\|\bbetaht -\betas\|_1 \lesssim_P \sigma_y\Bigl(\frac{\max_{j\in[p]}\Sigma_{j,j}^{-1}s^2\log p }{N}\Bigr)^{\frac12}+ \frac{\sqrt{\kappa}\max_{j\in[p]}(\Sigma_{j,j})^{\frac12}}{\lambda_{\min}(\Sigma)}\sigma_x^2\sigma_y\frac{s^2\log p}{n}.$
\ENUM
\end{theorem}

\begin{remark}
\label{rem:bbetaht-vs-lasso}
By Theorem \ref{thm:bbetaht-consistency}, when $m \lesssim \frac{n}{s^2 \log p},$ the variance term is dominant and the convergence rates given by the theorem simplify:
\BNUM
\item $\|\bbetaht- \betas\|_\infty \lesssim_P \bigl(\frac{ \log p}{N}\bigr)^{\frac12},$
\item $\|\bbetaht- \betas\|_2 \lesssim_P  \bigl(\frac{s\log p}{N}\bigr)^{\frac12},$
\item $\|\bbetaht- \betas\|_1 \lesssim_P \bigl(\frac{s^2\log p}{N}\bigr)^{\frac12}.$
\ENUM
The convergence rates for the centralized lasso estimator $\hat \beta $ are identical (modulo constants):
\BNUM
\item $\|\hat\beta - \betas\|_\infty \lesssim_P \bigl(\frac{ \log p}{N}\bigr)^{\frac12},$
\item $\|\hat\beta- \betas\|_2 \lesssim_P \bigl(\frac{s\log p}{N}\bigr)^{\frac12},$
\item $\|\hat\beta- \betas\|_1 \lesssim_P \bigl(\frac{s^2\log p}{N}\bigr)^{\frac12}.$
\ENUM
The estimator $\bbetaht$ matches the convergence rates of the centralized lasso in $\ell_1$, $\ell_2$, and $\ell_\infty$ norms. Furthermore, $\bbetaht$ can be evaluated in a communication-efficient manner by a one-shot averaging approach.
\end{remark}

We conduct a third set of simulations to study the effect of thresholding on the estimation error in $\ell_2$ norm. Figure \ref{fig:n-fixed-mse} compares the estimation error incurred by the averaged estimator with and without thresholding versus that of the centralized lasso. Since the averaged estimator is usually dense, its estimation error (in $\ell_2$ norm) is large compared to that of the centralized lasso. However, after thresholding, the averaged estimator performs comparably versus the centralized lasso.


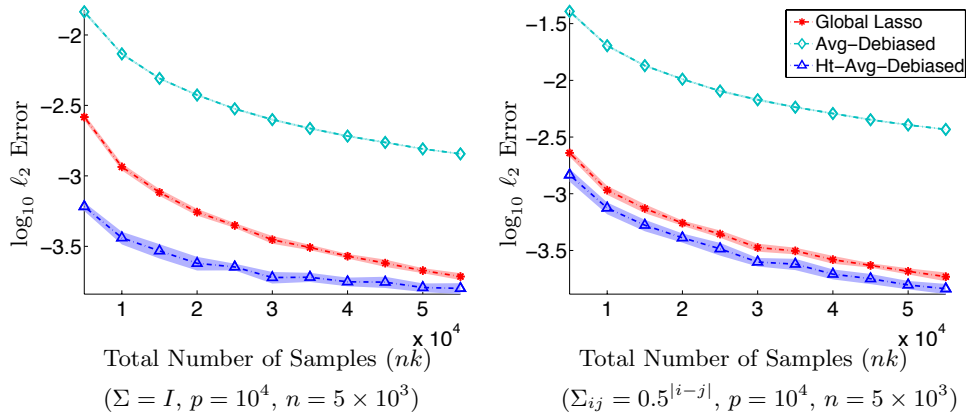
\begin{figure}
  \begin{tabular}{cc}
    \includegraphics[width=.45\textwidth]{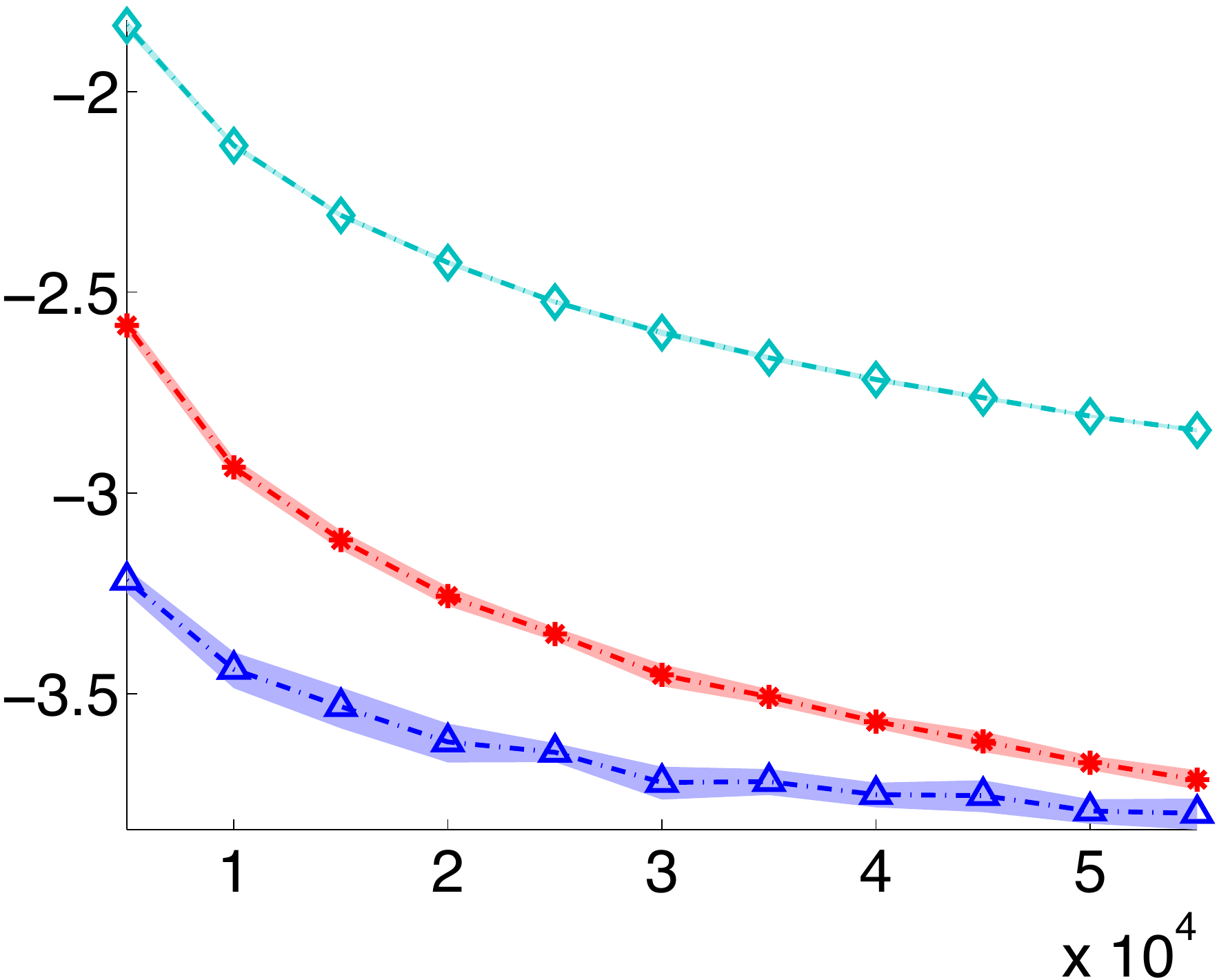}
    ~~~~~
    \includegraphics[width=.45\textwidth]{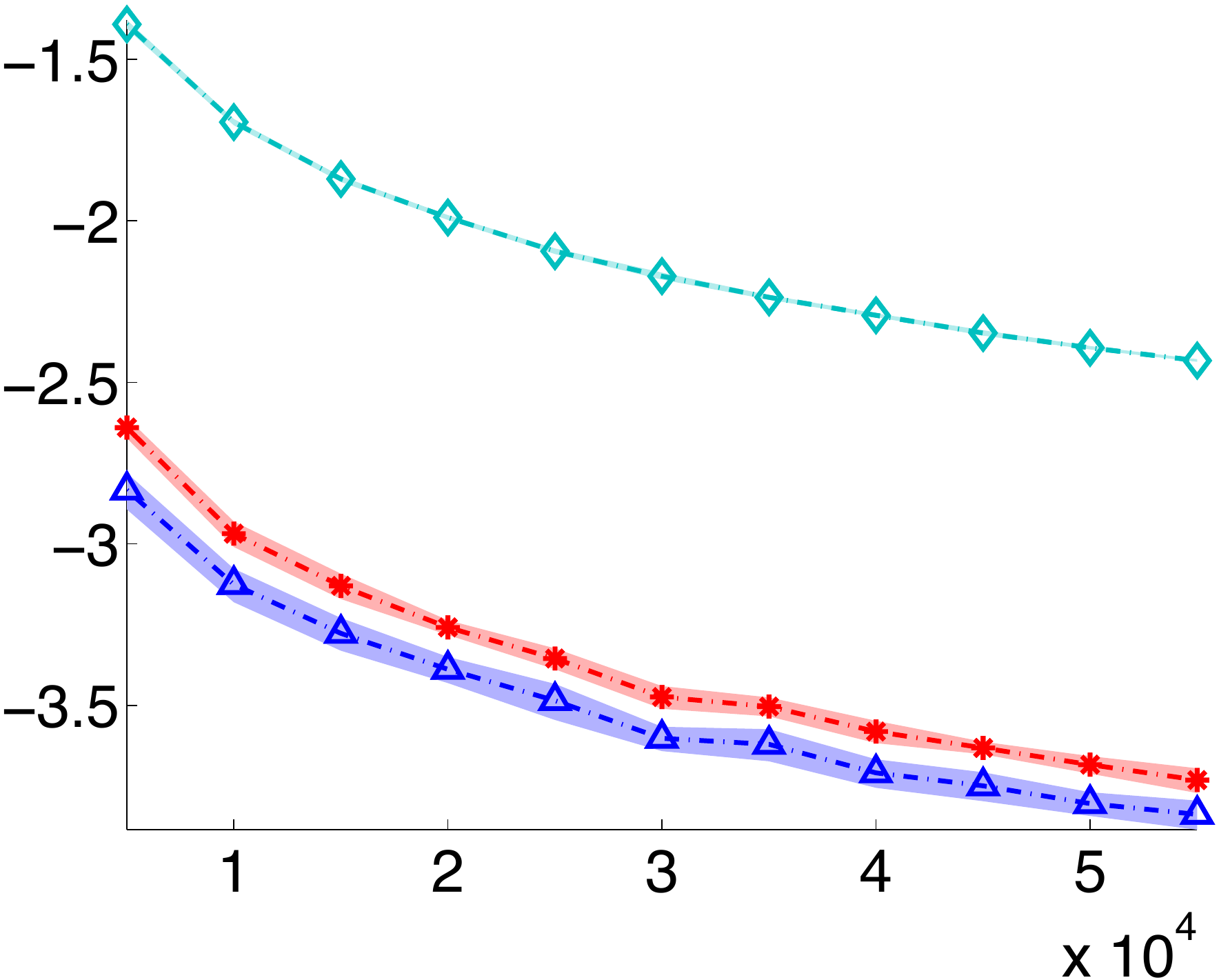}
     \hspace{-7.6em} \raisebox{11em}{\includegraphics[width=.2\textwidth]{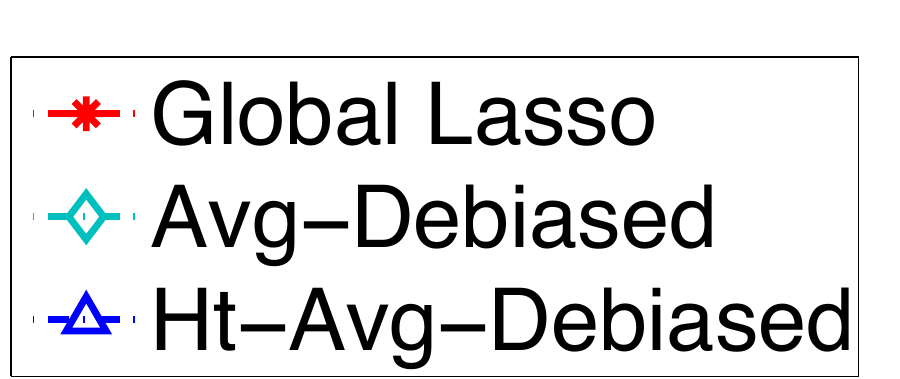}}  \\
   \end{tabular}
   \begin{picture}(0,18)(0,0)
   \put(-150,15){$\text{Total Number of Samples ($nk$)}$}
   \put(30,15){$\text{Total Number of Samples ($nk$)}$}
    \put(-150,0){($\Sigma = I$, $p =10^4$, $n= 5\times10^3$)}
   \put(20,0){ ($\Sigma_{ij} = 0.5^{\abs{i-j}}$, $p =10^4$, $n= 5\times10^3$)}
   \put(-185, 60){\rotatebox{90}{$\log_{10}$ $\ell_{2}$ Error}}
   \put(0, 60){\rotatebox{90}{$\log_{10}$ $\ell_{2}$ Error}}
   \end{picture}
  \caption{The estimation error (in $\ell_2$ norm) of the averaged estimator with and sans thresholding versus that of the centralized lasso when the predictors are Gaussian. In both settings, thresholding reduces the estimation error by order(s) of magnitude. Although the estimation error of the averaged estimator is large compared to that of the centralized lasso, the thresholded averaged estimator performs comparably, or even better than, the centralized lasso. }
  \label{fig:n-fixed-mse}
\end{figure}

\section{A distributed approach to debiasing}

The averaged estimator we studied has the form
\[
\bbeta = \frac1m\sum_{k=1}^m\hbeta_k + \hTheta_k X_k ^T(y-X_k\hbeta_k).
\]
The estimator requires each machine to form $\hTheta_k$ by the solution of \eqref{eq:javanmard-LDPE}. Since the dual of \eqref{eq:javanmard-LDPE} is an $\ell_1$-regularized quadratic program:
\BEQ
\minimize_{\gamma\in\reals^{p}}\frac12\gamma^T\hSigma_k\gamma - \hSigma_k\gamma + \delta\norm{\gamma}_1,
\label{eq:javanmard-LDPE-dual}
\EEQ
forming $\hTheta_k$ is (roughly speaking) $p$ times as expensive as solving the local lasso problem, making it the most expensive step (in terms of FLOPS) of evaluating the averaged estimator. To trim the cost of the debiasing step, we consider an estimator that forms only a single $\hTheta:$
\BEQ
\label{eq:parallel-avg-estimator}
\begin{aligned}
\tbeta =\frac1m\sum_{k=1}^m \hbeta_k +\frac{1}{N} \hTheta\sum_{k=1}^m X_k^T(y-X_k\hbeta_k).
\end{aligned}
\EEQ
To evaluate \eqref{eq:parallel-avg-estimator},
\BNUM
\item each machine sends $\hbeta_k$ and $\frac1nX_k ^T (y-X_k\hbeta_k)$ to a central server,
\item the central server forms $\frac1m\sum_{k=1}^m \hbeta_k$ and $\frac{1}{N}\sum_{k=1}^m X_k ^T (y-X_k\hbeta_k)$ and sends the averages to all the machines,
\item each machine, given the averages, forms $\frac{p}{m}$ rows of $\hTheta$ and debiases $\frac{p}{m}$ coefficients:
$$
\tbeta_j = \frac1m\sum_{k=1}^m\hbeta_j + \hTheta_{j,\cdot}\Bigl(\frac{1}{N}\sum_{k=1}^mX_k^T(y-X_k\hbeta_k)\Bigr),
$$
where $\hTheta_{j,\cdot}\in\reals^p$ is a row vector.
\ENUM
As we shall see, each machine can perform debiasing with only the data stored locally. Thus, forming the estimator \eqref{eq:parallel-avg-estimator} requires two rounds of communication.

The question that remains is how to form $\hTheta_{j,\cdot}.$ We consider an estimator proposed by \cite{van2013asymptotically}: nodewise regression on the predictors. For some $j \in [p]$ that machine $k$ is debiasing, the machine solves
\[
\hgamma_j := \argmin_{\gamma\in\reals^{p-1}}\frac{1}{2n}\|X_{k,j} - X_{k,-j}\gamma\|_2^2 + \lambda_j\|\gamma\|_1,\,j\in[p],
\]
where $X_{k,-j}\in\reals^{n\times(p-1)}$ is $X_k$ less its $j$-th column $X_{k,j}$. Implicitly, we are forming
\[
\hC:= \BMAT 1 & -\hgamma_{1,2} & \dots & -\hgamma_{1,p} \\
-\hgamma_{2,1} & 1 & \dots & -\hgamma_{2,p} \\
\vdots & \vdots & \ddots & \vdots \\
-\hgamma_{p,1} & -\hgamma_{p,2} & \dots &-\hgamma_{p,p} \EMAT,
\]
where the components of $\hgamma_j$ are indexed by $k\in\{1,\dots,j-1,j+1,\dots,p\}.$ We scale the rows of $\hC$ by $\diag\Bigl(\BMAT\htau_1,\dots,\htau_p\EMAT\Bigr)$, where
\[
\htau_j = \bigl(\frac1n\|X_j - X_{-j}\hgamma_j\|_2^2 + \lambda_j\|\hgamma_j\|_1\bigr)^\frac12,
\]
to form $\hTheta = \hT^{-2}\hC.$ Each row of $\hTheta$ is given by
\BEQ
\hTheta_{j,\cdot} = -\frac{1}{\htau_j^2}\BMAT\hgamma_{j,1} & \dots & \hgamma_{j,j-1} & 1 & \hgamma_{j,j+1} & \dots &\hgamma_{j,p}\EMAT.
\label{eq:vdG-LDPE}
\EEQ
Since $\hgamma_j$ and $\htau_j$ only depend on $X_k,$ they can be formed without any communication.

Before we justify the choice of $\hTheta$ theoretically, we mention that it is a approximate ``inverse'' of $\hSigma$ (in a component-wise sense). By the optimality conditions of nodewise regression,
\begin{align*}
\htau_j^2 &= \frac1n\|X_j - X_{-j}\hgamma_j\|_2^2 + \lambda_j\|\hgamma_j\|_1 \\
&= \frac1n\|X_j - X_{-j}\hgamma_j\|_2^2 + \frac1n (X_j - X_{-j}\hgamma_j)^TX_{-j}^T\hgamma_j \\
&= \frac1n X_j(X_j - X_{-j}\hgamma).
\end{align*}
Recalling the defintition of $\hTheta$, we have
\begin{gather*}
\frac1n\hTheta_{j,\cdot}X^TX_j = \frac{1}{\htau_j^2}\frac1n(X_j - \hgamma_j^TX_{-j})^TX_j = \lambda_j\text{ and } \\
\frac1n\|\hTheta_{j,\cdot}X^TX_{-j}\|_\infty = \frac{1}{\htau_j^2}\Bigl\|\frac1n(X_j - \hgamma_j^TX_{-j})^TX_{-j}\Bigr\|_\infty \le \frac{\lambda_j}{\htau_j^2}
\end{gather*}
for any $j\in[\,p\,]$. Thus
\BEQ
\label{eq:ext-kkt}
\max_{j\in[\,p\,]}\|\hTheta_{j,\cdot}\hSigma - e_j\|_\infty \le \frac{\lambda_j}{\htau_j^2}.
\EEQ

\cite{van2013asymptotically} show that when the rows of $X$ are \iid{} subgaussian random vectors and the precision matrix $\Sigma^{-1}$ is sparse, $\hTheta_{j,\cdot}$ converges to $\Sigma_j^{-1}$ at the usual convergence rate of the lasso. For completeness, we restate their result.

We consider a sequence of regression problems indexed by the sample size $N$, dimension $p$, sparsity $s_0$ that satisfies (A1), (A2), and (A3). As $N$ grows to infinity, both $p = p(N)$ and $s = s(N)$ may also grow as a function of $N.$ To keep notation manageable, we drop the index $N.$ We further assume
\BNUM
\item[(A4)] the covariance of the predictors (rows of $X$) has smallest eigenvalue $\lambda_{\min}(\Sigma) \sim \Omega(1)$ and largest diagonal entry $\max_{j\in[p]}\Sigma_{j,j}\sim O(1)$,
\item[(A5)] the rows of $\Sigma^{-1}$ are sparse: $\max_{j\in[p]} \frac{s_j^2\log p}{n} \sim o(1)$, where $s_j$ is the sparsity of $\Sigma_j^{-1}$.
\ENUM

\begin{lemma}[\cite{van2013asymptotically}, Theorem 2.4]
\label{lem:vdG-LDPE-consistency}
Under (A1)--(A5), \eqref{eq:vdG-LDPE} with suitable parameters $\lambda_j \sim \bigl(\frac{\log p}{n}\bigr)^{\frac12}$ satisfies
\[
\|\hTheta_{j,\cdot} - \Sigma_j^{-1}\|_1 \lesssim_P \biggl(\frac{s_j^2\log p}{n}\biggr)^\frac12\text{ for any }j\in[p].
\]
\end{lemma}

We show that the averaged estimator \eqref{eq:parallel-avg-estimator} matches the convergence rate of the centralized lasso.

\begin{theorem}
\label{thm:parallel-avg-estimator-consistency-1}
Under (A1)--(A5), \eqref{eq:parallel-avg-estimator}, where $\hTheta$ is given by \eqref{eq:vdG-LDPE}, with suitable parameters $\lambda_j, \lambda_k \sim \bigl(\frac{\log p}{n}\bigr)^{\frac12}$, $j\in[p]$, $k\in[m]$ satisfies
\[
\|\bbeta - \betas\|_\infty \lesssim_P \Bigl(\frac{\log p }{N}\Bigr)^{\frac12}+ \frac{s_{\max}\log p}{n},
\]
where $s_{\max} := \max\{s_0,s_1,\dots,s_p\}$.
\end{theorem}

\begin{proof}
We start by substituting the linear model into \eqref{eq:parallel-avg-estimator}:
\begin{align*}
\tbeta &=\frac1m\sum_{k=1}^m \hbeta_k - \hTheta\hSigma_k(\hbeta_k - \betas) + \frac1n\hTheta X_k^T\epsilon_k \\
&= \frac1m\sum_{k=1}^m \hbeta_k - \hTheta\hSigma_k(\hbeta_k - \betas) + \frac{1}{N}\hTheta X^T\epsilon.
\end{align*}
Subtracting $\betas$ and taking norms, we obtain
\BEQ
\|\tbeta - \betas\|_\infty \le \frac1m\sum_{k=1}^m \|(I - \hTheta\hSigma_k)(\hbeta_k - \betas)\|_\infty + \bigl\|\frac{1}{N}\hTheta X^T\epsilon\bigr\|_\infty.
\label{eq:parallel-avg-estimator-consistency-1-1}
\EEQ
By \cite{vershynin2010introduction}, Proposition 5.16, and Lemma \eqref{lem:c-var-bounded}, it is possible to show that
\[
\bigl\|\frac{1}{N}\hTheta X^T\epsilon\bigr\|_\infty \lesssim_P \Bigl(\frac{\log p}{N}\Bigr)^{\frac12}.
\]
We turn our attention to the first term in \eqref{eq:parallel-avg-estimator-consistency-1-1}. It's straightforward to see each term in the sum is bounded by
\begin{align*}
&\|(I - \hTheta\hSigma_k)(\hbeta_k - \betas)\|_\infty \\
&\pc\le \|(I - \Sigma^{-1}\hSigma_k)(\hbeta_k - \betas)\|_\infty + \|(\Sigma^{-1} - \hTheta)\hSigma_k(\hbeta_k - \betas)\|_\infty \\
&\textstyle\pc\le \max_{j\in[p]}\|e_j^T - \Sigma_j^{-1}\hSigma_k\|_\infty\|\hbeta_k - \betas\|_1 + \|\Sigma_j^{-1} - \hTheta_{j,\cdot}\|_1\|\hSigma_k(\hbeta_k - \betas)\|_\infty.
\end{align*}
We put the pieces together to deduce each term is $O\bigl(\frac{s_{\max} \log p}{n}\bigr):$
\BNUM
\item By Lemmas \ref{lem:rudelson-zhou-RE}, \ref{lem:lasso-consistency}, \ref{lem:c-bias-bounded}, $\|\hbeta_k - \betas\|_1 \lesssim_P \sqrt{s_0}\lambda_k$.
\item By Lemma \ref{lem:vdG-LDPE-consistency}, $\|\Sigma_j^{-1} - \hTheta_{j,\cdot}\|_1 \lesssim_P s_j\bigl(\frac{\log p}{n}\bigr)^\frac12$.
\item By the triangle inequality,
\[
\|\hSigma_k(\hbeta_k - \betas)\|_\infty \le \Bigl\|\frac1nX_k^T(y_k - X_k\hbeta_k)\Bigr\|_\infty + \Bigl\|\frac1nX_k^T\epsilon_k\Bigr\|_\infty.
\]
By the optimality conditions of the (local) lasso estimators, the first term is $\lambda_k$, and it is possible to show, by Lemma \ref{lem:c-var-bounded} and \cite{vershynin2010introduction}, Proposition 5.16, that the second term is $O_P\bigl(\bigl(\frac{\log p}{n}\bigr)^\frac12\bigr).$
\ENUM
Since $\lambda_k \sim \bigl(\frac{\log p}{n}\bigr)^{\frac12}$, by a union bound over $k\in[m],$ we obtain
\[
\|\bbeta - \betas\|_\infty \sim O_P\Bigl(\Bigl(\frac{\log p }{N}\Bigr)^{\frac12}+ \frac{s_{\max}\log p}{n}\biggr),
\]
where $s_{\max} := \max\{s_0,s_1,\dots,s_p\}$.
\end{proof}

By combining the Lemma \ref{lem:thresh-consistency} with Theorem \ref{thm:parallel-avg-estimator-consistency-1}, we can show that $\tbetaht := \HT(\tbeta,t)$ for an appropriate threshold $t$ converges to $\betas$ at the same rates as the centralized lasso.

\begin{theorem}
\label{thm:tbetaht-consistency-1}
Under the conditions of Theorem \ref{thm:parallel-avg-estimator-consistency-1}, hard-thresholding $\tbeta$ at $t\sim\bigl(\frac{\log p }{N}\bigr)^{\frac12} + \frac{s_{\max} \log p}{n}$ gives
\BNUM
\item $\|\tbetaht- \betas\|_\infty \lesssim_P \bigl(\frac{\log p }{N}\bigr)^{\frac12} + \frac{s_{\max}\log p}{n},$
\item $\|\tbetaht - \betas\|_2 \lesssim_P \bigl(\frac{s_0\log p }{N}\bigr)^{\frac12} + \frac{\sqrt{s_0}s_{\max} \log p}{n},$
\item $\|\tbetaht -\betas\|_1 \lesssim_P \bigl(\frac{s_0^2\log p }{N}\bigr)^{\frac12} + \frac{s_0s_{\max} \log p}{n}.$
\ENUM
\end{theorem}

Theorem \ref{thm:tbetaht-consistency-1} shows that for $m \lesssim \frac{n}{s_{\max}^2 \log p},$ the variance term is dominant, so the convergence rates simplify:
\BNUM
\item $\|\tbetaht- \betas\|_\infty \lesssim_P \bigl(\frac{\log p }{N}\bigr)^{\frac12},$
\item $\|\tbetaht- \betas\|_2 \lesssim_P  \bigl(\frac{s_{\max}\log p }{N}\bigr)^{\frac12},$
\item $\|\tbetaht- \betas\|_1 \lesssim_P \bigl(\frac{s_{\max}^2\log p }{N}\bigr)^{\frac12}.$
\ENUM
Thus, estimator $\tbetaht$ shares the advantages of $\bbetaht$ over the centralized lasso (\cf{} Remark \ref{rem:bbetaht-vs-lasso}). It also achieves computational gains over $\bbetaht$ by amorti\-zing the cost of debiasing across $m$ machines.

\section{Averaging debiased $\ell_1$ regularized M-estimators}

The distri\-buted approach to debiasing extends readily to $\ell_1$ regularized M-estimators. As before, we are given $N$ pairs $(x_i,y_i)$ stored on $m$ machines. Let $\rho(y_i,a)$ be a loss function function, which is convex in $a$, and $\drho$, $\ddrho$ be its derivatives with respect to $a$. That is
\[
\drho(y,a) = \frac{d}{da}\rho(y,a),\quad\ddrho(y,a) = \frac{d^2}{da^2}\rho(y,a).
\]
We define $\ell_k(\beta) = \frac1n\sum_{i=1}^n\rho(y_i,x_i^T\beta)$, where the sum is only over the pairs on machine $k$. The averaged estimator is
\BEQ
\bbeta := \frac1m\sum_{k=1}^m\hbeta_k + \hTheta\Bigl(\frac1m\sum_{k=1}^m\nabla\ell_k(\hbeta_k)\Bigr),
\label{eq:avg-glm-estimator}
\EEQ
where $\hbeta_k$ is the local $\ell_1$ regularized M-estimator: $\hbeta_k := \argmin_{\beta\in\reals^p}\ell_k(\beta) + \lambda_k\|\beta\|_1$. As before, we form $\hTheta$ by nodewise regression on the weighted design matrix $X_{\hbeta_k} := W_{\hbeta_k}X_k$, where $W_{\hbeta_k}$ is diagonal and its diagonal entries are
\[
\bigl(W_{\hbeta_k}\bigr)_{i,i} := \ddrho(y_i,x_i^T\hbeta_k)^{\frac12}.
\]
That is, for some $j \in [p]$ that machine $k$ is debiasing, the machine solves
\[
\hgamma_j := \argmin_{\gamma\in\reals^{p-1}}\frac{1}{2n}\|X_{\hbeta_k,j} - X_{\hbeta_k,-j}\gamma\|_2^2 + \lambda_j\|\gamma\|_1,\,j\in[p],
\]
and forms
\[
\hTheta_{j,\cdot} = -\frac{1}{\htau_j^2}\BMAT\hgamma_{j,1} & \dots & \hgamma_{j,j-1} & 1 & \hgamma_{j,j+1} & \dots &\hgamma_{j,p}\EMAT,
\]
where
\[
\htau_j = \bigl(\frac1n\|X_{\hbeta_k,j} - X_{\hbeta_k,-j}\hgamma_j\|_2^2 + \lambda_j\|\hgamma_j\|_1\bigr)^\frac12.
\]

We assume
\BIT
\item[(B1)] the pairs $\{(x_i,y_i)\}_{i\in[N]}$ are \iid{}; the predictors are boun\-ded:
\[\textstyle
\max_{i\in[N]}\|x_i\|_\infty \lesssim 1;
\]
the projection of $X_{\betas,j}$ on $\cR(X_{\betas,-j})$ in the $\Expect\left[\nabla^2\ell_k(\betas)\right]$ inner product is bounded: $\|X_{\betas,-j}\gamma_{\betas,j}\|_\infty\lesssim 1$ for any $j\in[\,p\,]$, where
\[
\gamma_{\betas,j} := \argmin_{\gamma\in\reals^{p-1}}\Expect\left[\|X_{\betas,j} - X_{\betas,-j}\gamma\|_2^2\right].
\]
\item[(B2)] the rows of $\Expect\left[\nabla^2\ell_k(\betas)\right]^{-1}$ are sparse: $\max_{j\in[p]} \frac{s_j^2\log p}{n} \sim o(1)$, where $s_j$ is the sparsity of $\bigl(\Expect\left[\nabla^2\ell_k(\betas)\right]^{-1}\bigr)_{j,\cdot}$.
\item[(B3)] the smallest eigenvalue of $\Expect\left[\nabla^2\ell_k(\betas)\right]$ is bounded away from zero and its entries are bounded.
\item[(B4)] for any $\beta$ such that $\|\beta - \betas\|_1 \le \delta$ for some $\delta > 0$, the diagonal entries of $W_{\beta}$ stays away from zero, and
\[
|\ddrho(y,x^T\beta) - \ddrho(y,x^T\betas)| \le |x^T(\beta - \betas)|.
\]
\item[(B5)] we have $\frac1n\|X_k(\hbeta_k - \betas)\|_2^2 \lesssim_P s_0\lambda_k^2$ and $\|\hbeta_k - \betas\|_1 \lesssim_P s_0\lambda_k$.
\item[(B6)] the derivatives $\drho(y,a)$, $\ddrho(y,a)$ is locally Lipschitz:
\[\textstyle
\max_{i\in[N]}\sup_{|a,a' - x_i^T\betas| \le \delta}\sup_y \frac{|\ddrho(y,a) - \ddrho(y,a')|}{|a - a'|} \le K\text{ for some }\delta > 0.
\]
Further,
\begin{gather*}
\textstyle\max_{i\in[N]}\sup_y|\drho(y,x_i^T\beta)| \sim O(1), \\
\textstyle\max_{i\in[N]}\sup_{|a - x_i^T\betas| \le \delta}\sup_y|\ddrho(y,a)| \sim O(1).
\end{gather*}
\item[(B7)] the diagonal entries of
\[
\Expect\bigl[\nabla^2\ell_k(\betas)\bigr]^{-1}\Expect\bigl[\nabla\ell_k(\betas)\nabla\ell_k(\betas)^T\bigr]\Expect\bigl[\nabla^2\ell_k(\betas)\bigr]^{-1}
\]
are bounded.
\EIT


Assumption (B5) not necessary; it is implied by the other assumptions. We refer to \cite{buhlmann2011statistics}, Chapter 6 for the details. Here we state it as an assumption to simplify the exposition. We show the averaged estimator \eqref{eq:avg-glm-estimator} achieves the convergence rate of the centralized $\ell_1$-regularized M-estimator.

\begin{theorem}
\label{thm:avg-glm-estimator-consistency}
Under (B1)--(B7), \eqref{eq:avg-glm-estimator} with suitable parameters \\ $\lambda_j,\lambda_k \sim \bigl(\frac{\log p}{n}\bigr)^{\frac12}$, $j\in[p]$, $k\in[m]$ satisfies
\BEQ
\|\bbeta - \betas\|_\infty \lesssim_P \Bigl(\frac{\log p }{N}\Bigr)^{\frac12}+ \frac{s_{\max}\log p}{n},
\label{eq:avg-glm-estimator-consistency}
\EEQ
where $s_{\max} := \max\{s_0,s_1,\dots,s_p\}$.
\end{theorem}

\begin{proof}
The averaged estimator is given by
\[
\bbeta - \betas = \frac1m\sum_{k=1}^m \hbeta_k - \hTheta\nabla\ell_k(\hbeta_k)(\hbeta_k - \betas) - \betas.
\]
By the smoothness of $\rho$,
\[
\drho(y_i,x_i^T\hbeta_k) = \drho(y_i,x_i^T\betas) + \ddrho(y_i,\ta_i)x_i^T(\hbeta_k - \betas),
\]
where $\ta_i$ is a point between $x_i^T\hbeta_k$ and $x_i^T\betas$. Thus
\begin{align*}
\bbeta - \betas &= \frac1m\sum_{k=1}^m \hbeta_k - \hTheta(\nabla\ell_k(\betas) + Q_k(\hbeta_k - \betas)) - \betas \\
&= -\hTheta{\textstyle\bigl(\frac1m\sum_{k=1}^m\nabla\ell_k(\betas)\bigr)} + \frac1m\sum_{k=1}^m\bigl(I -\hTheta Q_k\bigr)(\hbeta_k - \betas).
\end{align*}
where $Q_k = \frac1n\sum_{i=1}^n\ddrho(y_i,\ta_i)x_ix_i^T$, where the sum is over the data points on machine $k$. Taking norms, we obtain
\[
\|\bbeta - \betas\|_\infty \le \bigl\|\hTheta{\textstyle\bigl(\frac1m\sum_{k=1}^m\nabla\ell_k(\betas)\bigr)}\bigr\|_\infty + \frac1m\sum_{k=1}^m\bigl\|\bigl(I - \hTheta Q_k\bigr)(\hbeta_k - \betas)\bigr\|_\infty.
\]

It is possible to show that $\bigl\|\hTheta\bigl(\frac1m\sum_{k=1}^m\nabla\ell_k(\betas)\bigr)\bigr\|_\infty \lesssim_P\bigl(\frac{\log p}{N}\bigr)^{\frac12}$, which corresponds to the first term in \eqref{eq:avg-glm-estimator-consistency}. We refer to \cite{buhlmann2011statistics}, Chapter 6 for the details.

We turn our attention to the second term. By the triangle inequality,
\begin{align*}
&\|(I - \hTheta Q_k)(\hbeta_k - \betas)\|_\infty \\
&\pc\le \bigl\|\bigl(I - \hTheta\nabla^2\ell_k(\hbeta_k)\bigr)(\hbeta_k - \betas)\bigr\|_\infty + \bigl\|\hTheta(\nabla^2\ell_k(\hbeta_k) - Q_k)(\hbeta_k - \betas)\bigr\|_\infty \\
&\pc\le {\textstyle\max_{j\in[p]}\bigl\|e_j^T - \hTheta_{j,\cdot}\nabla^2\ell_k(\hbeta_k)\bigr\|_\infty}\|\hbeta_k - \betas\|_1 \\
&\pc\pc+ \frac1n\sum_{i=1}^n\|\hTheta x_i\|_\infty\bigl|\ddrho(y_i,x_i^T\hbeta_k) - \ddrho(y_i,\ta_i)x_i^T(\hbeta_k - \betas)\bigr|.
\end{align*}
We proceed term by term. By \eqref{eq:ext-kkt},
\[
{\textstyle\max_{j\in[p]}\bigl\|e_j^T - \hTheta_{j,\cdot}\nabla^2\ell_k(\hbeta_k)\bigr\|_\infty} \le \frac{\lambda_j}{\htau_j^2} \lesssim \frac{1}{\htau_j^2}\Bigl(\frac{\log p}{n}\Bigr)^{\frac12}.
\]
By \cite{van2013asymptotically}, Theorem 3.2,
\[
|\htau_j^2 - \tau_j^2| \lesssim_P \Bigl(\frac{\max\{s_0,s_j\}\log p}{n}\Bigr)^{\frac12}
\]
Thus $\max_{j\in[p]}\bigl\|e_j^T - \hTheta_{j,\cdot}\nabla^2\ell_k(\hbeta_k)\bigr\|_\infty \lesssim_P \bigl(\frac{\log p}{n}\bigr)^{\frac12}$ and, by (B5),
\[
{\textstyle\max_{j\in[p]}\bigl\|e_j^T - \hTheta_{j,\cdot}\nabla^2\ell_k(\hbeta_k)\bigr\|_\infty}\|\hbeta_k - \betas\|_1 \lesssim_P \frac{s_{\max}\log p}{n}.
\]
We turn our attention to the second term. We have $\|\hTheta x_i\|_\infty\lesssim_P 1$ because
\begin{align*}
\|\hTheta x_i\|_\infty &\le {\textstyle\max_{j\in[p]}\|\hTheta_{j,\cdot}X_k^T\|_\infty} \lesssim {\textstyle\max_{j\in[p]}\|}\hTheta_{j,\cdot}X_{k,\betas}^T\|_\infty \\
&\le {\textstyle\max_{j\in[p]}}\,\frac{1}{\htau_j^2}\|(X_{k,\betas})_j - (X_{k,\betas})_{-j}\hgamma_j\|_\infty.
\intertext{Again, by \cite{van2013asymptotically}, Theorem 3.2, }
&\lesssim_P {\textstyle\max_{j\in[p]}}\,\frac{1}{\tau_j^2}\|(X_{k,\betas})_j - (X_{k,\betas})_{-j}\hgamma_j\|_\infty \\
&\lesssim_P {\textstyle\max_{j\in[p]}}\,\frac{1}{\tau_j^2}\|(X_{k,\betas})_j - (X_{k,\betas})_{-j}\gamma_j\|_\infty \\
&\pc+ \frac{1}{\tau_j^2}\|(X_{k,\betas})_j\|_\infty\|(\hgamma_j - \gamma_j)\|_1.
\intertext{which, by (B1) and \cite{van2013asymptotically}, Theorem 3.2,}
&\lesssim_P 1 + \frac{s_j\log p}{n}.
\end{align*}
Thus
\begin{align*}
&\frac1n\sum_{i=1}^n\|\hTheta x_i\|_\infty\bigl|\ddrho(y_i,x_i^T\hbeta_k) - \ddrho(y_i,\ta_i)x_i^T(\hbeta_k - \betas)\bigr| \\
&\pc\lesssim_P\frac1n\sum_{i=1}^n\bigl|\ddrho(y_i,x_i^T\hbeta_k) - \ddrho(y_i,\ta_i)x_i^T(\hbeta_k - \betas)\bigr|,
\intertext{which, by (B5) and (B6), is at most}
&\pc\lesssim \frac1n\|X_k(\hbeta_k - \betas)\|_2^2 \lesssim_P \frac{s_0\log p}{n}.
\end{align*}
We put the pieces together to deduce $\frac1m\sum_{k=1}^m\bigl\|\bigl(I - \hTheta Q_k\bigr)(\hbeta_k - \betas)\bigr\|_\infty \lesssim_P \frac{s_{\max}\log p}{n}$.
\end{proof}

By combining the Lemma \ref{lem:thresh-consistency} with Theorem \ref{thm:parallel-avg-estimator-consistency-1}, we can show that $\tbetaht := \HT(\tbeta,t)$ for an appropriate threshold $t$ converges to $\betas$ at the same rates as the centralized $\ell_1$-regularized M-estimator.

\begin{theorem}
\label{thm:thresh-avg-glm-estimator-consistency}
Under the conditions of Theorem \ref{thm:avg-glm-estimator-consistency}, hard-thresholding $\tbeta$ at $t\sim\bigl(\frac{\log p }{N}\bigr)^{\frac12} + \frac{\max_{j\in[p]}s_j \log p}{n}$ gives
\BNUM
\item $\|\tbetaht- \betas\|_\infty \lesssim_P \bigl(\frac{\log p }{N}\bigr)^{\frac12} + \frac{s_{\max} \log p}{n},$
\item $\|\tbetaht - \betas\|_2 \lesssim_P \bigl(\frac{s_0\log p }{N}\bigr)^{\frac12} + \frac{\sqrt{s_0}s_{\max} \log p}{n},$
\item $\|\tbetaht -\betas\|_1 \lesssim_P \bigl(\frac{s_0^2\log p }{N}\bigr)^{\frac12} + \frac{s_0\max_{j\in[p]}s_j \log p}{n}.$
\ENUM
\end{theorem}

Assuming $s_0\sim s_{\max}$, Theorem \ref{thm:thresh-avg-glm-estimator-consistency} shows when $m \lesssim \frac{n}{s_0^2 \log p}$, the vari\-ance term is dominant, so the convergence rates simplify to
\BNUM
\item $\|\tbetaht- \betas\|_\infty \lesssim_P \bigl(\frac{\log p }{N}\bigr)^{\frac12},$
\item $\|\tbetaht- \betas\|_2 \lesssim_P  \bigl(\frac{s_0\log p }{N}\bigr)^{\frac12},$
\item $\|\tbetaht- \betas\|_1 \lesssim_P \bigl(\frac{s_0^2\log p }{N}\bigr)^{\frac12}.$
\ENUM

\section{Summary and discussion}

We devised a communication-efficient approach to distributed sparse regression in the high-dimensional setting. The key idea is first ``debiasing'' local lasso estimators, and then averaging the debiased estimators. We show that as long as the data is not split across too many machines, the averaged estimator achieves the convergence rate of the centralized lasso estimator. In the appendix, we show that by foregoing consistency in the $\ell_\infty$ norm, it is possible to further reduce the sample complexity of the averaged estimator to that of the centralized lasso estimator. Further, the distributed approach to debiasing extends readily to other $\ell_1$ regularized M-estimators. In concurrent work, the approach of averaging debiased M-estimators was proposed by \cite{battey2015splitotic} for high-dimensional inference.


In recent years, there has a been a flurry of work on establishing communication lower bounds for mean estimation in the Gaussian distribution. In other words, they establish the minimum communication $C$  needed to obtain $\ell_2 ^2$ risk $R$ , where $ \| \hat \beta - \betas\|^2_2 \le R$ \citep{duchi2014optimality, garg2014lower}. These results are not directly applicable to sparse linear regression, since they do not impose sparsity on the mean. In \cite{braverman2015communication}, the authors established that to obtain risk $R\le \frac{ s \log p }{N}$ at least $\Omega\bigl(\frac{m \min(n,p)}{\log p}\bigr)$ bits of communication is required. Our approach communicates $\tilde O(mp) $ bits to achieve risk of $\frac{s \log p}{N}$, so is communication-optimal when $p \lesssim n $.

\appendix
\section{Proofs of Lemmas}

\begin{proof}[Proof of Lemma \ref{lem:GC}]
Let $z_i = \Sigma^{-\frac12}x_i.$ The generalized coherence between $X$ and $\Sigma^{-1}$ is given by
\[
|||\Sigma^{-1}\hSigma - I|||_\infty = |||\frac1n\sum_{i=1}^n (\Sigma^{-\frac12}z_i)(\Sigma^{\frac12}z_i)^T - I|||_\infty.
\]
Each entry of $\frac1n\sum_{i=1}^n (\Sigma^{-\frac12}z_i)(\Sigma^{\frac12}z_i)^T - I$ is a sum of independent subexponential random variables. Their subexponential norms are bounded by
\[
\|(\Sigma^{-\frac12}z_i)_j(\Sigma^{\frac12}z_i)_k - \delta_{j,k}\|_{\psi_1} \le 2\|(\Sigma^{-\frac12}z_i)_j(\Sigma^{\frac12}z_i)_k\|_{\psi_1}.
\]
Recall for any two subgaussian random variables $X,Y,$ we have
\[
\norm{XY}_{\psi_1} \le 2\norm{X}_{\psi_2}\norm{Y}_{\psi_2}.
\]
Thus
\begin{align*}
\|(\Sigma^{-\frac12}z_i)_j(\Sigma^{\frac12}z_i)_k - \delta_{j,k}\|_{\psi_1} &\le 4\|(\Sigma^{-\frac12}z_i)_j\|_{\psi_2}\|(\Sigma^{\frac12}z_i)_k\|_{\psi_2} \le 4\sqrt{\kappa}\sigma_x^2,
\end{align*}
where $\sigma_x = \|z_i\|_{\psi_2}.$ By a Bernstein-type inequality,
\[
\Pr\Bigl(\frac1n\sum_{i=1}^n (\Sigma^{-\frac12}z_i)_j(\Sigma^{\frac12}z_i)_k - \delta_{j,k} \ge t\Bigr) \le 2e^{-c_1\min\{\frac{nt^2}{\tsigma_x^4},\frac{nt}{\tsigma_x^2}\}},
\]
where $c_1 > 0$ is a universal constant and $\tsigma_x^2 := 4\sqrt{\kappa}\sigma_x^2.$ Since $\tsigma_x^4 n > \log p,$ we set $t = \frac{2\tsigma_x^2}{\sqrt{c_1}}\bigl(\frac{\log p}{n}\bigr)^{\frac12}$ to obtain
\[
\Pr\Bigl(\frac1n\sum_{i=1}^n (\Sigma^{-\frac12}z_i)_j(\Sigma^{\frac12}z_i)_k - \delta_{j,k} \ge  \frac{2\tsigma_x^2}{\sqrt{c_1}}\Bigl(\frac{\log p}{n}\Bigr)^{\frac12}\Bigr) \le 2p^{-4}.
\]
We obtain the stated result by taking a union bound over the $p^2$ entries of $\frac1n\sum_{i=1}^n (\Sigma^{-\frac12}z_i)(\Sigma^{\frac12}z_i)^T - I.$
\end{proof}

\begin{proof}[Proof of Lemma \ref{lem:lambda-large}]
By \cite{vershynin2010introduction}, Proposition 5.10,
\[
\Pr\Bigl(\frac1n|x_j^T\epsilon| > t\Bigr) \le e\exp\Bigl(-\frac{c_2n^2t^2}{\sigma_y^2\|x_j^T\|_2^2}\Bigr) \le e\exp\Bigl(-\frac{c_2n^2t^2}{\sigma_y^2\max_{j\in[p]}\hSigma_{j,j}}\Bigr).
\]
We take a union bound over the $p$ components of $\frac1nX^T\epsilon$ to obtain
\[
\Pr\Bigl(\frac1n\|X^T\epsilon\|_\infty > t\Bigr) \le e\exp\Bigl(-\frac{c_2n^2t^2}{\sigma_y^2\max_{j\in[p]}\hSigma_{j,j}} + \log p\Bigr).
\]
We set $\lambda = \max_{j\in[p]}\hSigma_{j,j}^{\frac12}\sigma_y\bigl(\frac{3 \log p}{c_2n}\bigr)^{\frac12}$ to obtain the desired conclusion.
\end{proof}

\begin{proof}[Proof of Lemma \ref{lem:debiased-lasso-consistency}]
We start by substituting in the linear model into \eqref{eq:debiased-lasso}:
\[
\hbetad = \hbeta + \frac1n\hTheta X^T(y - X\hbeta) = \betas + M\hSigma(\betas - \hbeta) + \frac1nMX^T\epsilon.
\]
By adding and subtracting $\hDelta = \betas - \hbeta,$ we obtain
\[
\hbetad = \betas + \frac1n\hTheta X^T(y - X\hbeta) = \betas + (M\hSigma - I)(\betas - \hbeta) + \frac1nMX^T\epsilon.
\]
We obtain the expression of $\hbetad$ by setting $\hDelta = (M\hSigma - I)(\betas - \hbeta).$

To show $\|\hDelta\|_\infty \le \frac{3\delta}{\mu}s\lambda,$ we apply H\"{o}lder's inequality to each component of $\hDelta$ to obtain
\BEQ\textstyle
|(M\hSigma - I)(\betas - \hbeta)| \le \max_j\|\hSigma m_j^T - e_j\|_\infty\|\hbeta - \betas\|_1 \le \delta\|\hbeta - \betas\|_1,
\label{eq:debiased-lasso-consistency-2}
\EEQ
where $\delta$ is the generalized incoherence between $X$ and $M.$
By Lemma \ref{lem:lasso-consistency}, $\|\hbeta - \betas\|_1 \le \frac{3}{\mu}s\lambda.$ We combine the bound on $\|\hbeta - \betas\|_1$ with \eqref{eq:debiased-lasso-consistency-2} to obtain the stated bound on $\|\hDelta\|_\infty.$
\end{proof}

\begin{proof}[Proof of Lemma \ref{lem:averaged-debiased-lasso-consistency}]
By Lemma \ref{lem:debiased-lasso-consistency},
\[
\bbeta-\beta^\star = \frac{1}{N}\sum_{k=1}^m \hTheta_k X_k ^T\epsilon_k+ \frac1m \sum_{k=1}^m \hDelta_k.
\]
We take norms to obtain
\[
\|\bbeta - \betas\|_\infty \le \Bigl\|\frac{1}{N}\sum_{k=1}^m \hTheta_k X_k ^T\epsilon_k\Bigr\|_\infty + \frac1m \sum_{k=1}^m \|\hDelta_k\|_\infty.
\]
We focus on bounding the first term. Let $a_j^T := e_j^T\BMAT \hTheta_1X_1^T  & \dots & \hTheta_mX_m^T\EMAT$. By \cite{vershynin2010introduction}, Pro\-position 5.10,
\[
\Pr\Bigl(\bigl|\frac{1}{N}a_j^T \epsilon\bigr| > t \Bigr) \le e\exp\Bigl(- \frac{c_2 N^2 t^2}{ \|a_j\|_2^2\sigma_y^2}\Bigr)
\]
for some universal constant $c_2 > 0.$ Further,
\[
\|a_j\|_2 ^2 = \sum_{k=1}^m \|X_k \hTheta_k ^T e_j\|_2 ^2 = n\sum_{k=1}^m  \bigl(\hTheta_k \hSigma_k \hTheta_k ^T\bigr)_{j,j}\le c_{\Omega}N,
\]
where $c_{\Omega} := \max_{j\in[p],k\in[m]}\bigl( \hTheta_k \hSigma_k \hTheta_k^T\bigr)_{j,j}.$ By a union bound over $j\in[p],$
\[
\Pr\Bigl({\textstyle\max_{j\in[p]}\bigl|}\frac{1}{N}a_j^T \epsilon\bigr| >t\Bigr) \le e\exp\Bigl(- \frac{c_2 N t^2}{c_{\Omega}\sigma_y^2} +\log p\Bigr).
\]
We set $t = \sigma_y\bigl(\frac{2c_{\Omega}\log p }{c_2N}\bigr)^{\frac12}$ to deduce
\[\textstyle
\Pr\Bigl({\textstyle\max_{j\in[p]}\bigl|}\frac{1}{N}a_j^T \epsilon\bigr| \ge \sigma_y\bigl(\frac{ 2c_{\Omega}\log p }{c_2N}\bigr)^{\frac12}\Bigr) \le ep^{-1}.
\]

We turn our attention to bounding the second term. By Lemma \ref{lem:lambda-large} and a union bound over $j\in[p]$, when we set
\[
\lambda_1 = \dots = \lambda_m = \lambda := {\textstyle\max_{j\in[p],k\in[m]}((\hSigma_k\bigr)_{j,j})^{\frac12}}\sigma_y\Bigr(\frac{3\log p}{c_2n}\Bigr)^{\frac12},
\]
we have $\frac1n\|X_k^T\epsilon\|_\infty \le \lambda$ for any $k\in[m]$ with probability at least $1 - \frac{em}{p^2} \ge 1 - ep^{-1}.$ By Lemma \ref{lem:debiased-lasso-consistency}, when
\BNUM
\item $\{\hSigma_k\}_{k\in[m]}$ satisfy the RE condition on $\cC^*$ with constant $\mu_l,$
\item $\{(\hSigma_k,\hTheta_k)\}_{k\in[m]}$ have generalized incoherence $c_{\GC}\bigl(\frac{\log p}{n}\bigr)^{\frac12},$
\ENUM
 the second term is at most $\frac{3\sqrt{3}}{\sqrt{c_2}}\frac{c_{\GC}c_{\Sigma}}{\mu_l}\sigma_y\frac{s\log p}{n}$. We put the pieces together to obtain
\[
\|\bbeta - \betas\|_\infty \le \sigma_y\Bigl(\frac{2c_{\Omega}\log p }{c_2N}\Bigr)^{\frac12}+ \frac{3\sqrt{3}}{\sqrt{c_2}}\frac{c_{\GC}c_{\Sigma}}{\mu_l}\sigma_y\frac{s\log p}{n},
\]
\end{proof}

\begin{proof}[Proof of Lemma \ref{lem:c-var-bounded}]
We express
\[
\Sigma_{j,\cdot}^{-1}\hSigma\Sigma_{j,\cdot}^{-1} = \Sigma_{j,\cdot}^{-1}\hSigma\Sigma_{j,\cdot}^{-1} - \Sigma_{j,j}^{-1} + \Sigma_{j,j}^{-1} = \frac1n\sum_{i=1}^n(x_i^T\Sigma_{\cdot,j})^2 - \Sigma_{j,j}^{-1} + \Sigma_{j,j}^{-1}.
\]
Since the subgaussian norm of $z_i = \Sigma^{-\frac12}x_i$ is $\sigma_x,$ $x_i^T\Sigma_{\cdot,j}$ is also subgaussian with subgaussian norm bounded by
\[
\|x_i^T\Sigma_{\cdot,j}\|_{\psi_2} \le \|\Sigma^{\frac12}z_i\|_{\psi_2}\|\Sigma_{\cdot,j}\|_2\le \sigma_x(\Sigma_{j,j})^{\frac12}.
\]
We recognize $\frac1n\sum_{i=1}^n(x_i^T\Sigma_{\cdot,j})^2 - \Sigma_{j,j}^{-1}$ as a sum of \iid{} subexponential random variables with subexponential norm bounded by
\[
\|(x_i^T\Sigma_{\cdot,j})^2 - \Sigma_{j,j}^{-1}\|_{\psi_1} \le 2\|(x_i^T\Sigma_{\cdot,j})^2\|_{\psi_1} \le 4\|x_i^T\Sigma_{\cdot,j}\|_{\psi_2}^2 \le 4\sigma_x^2\Sigma_{j,j}^{-1}.
\]
By \cite{vershynin2010introduction}, Proposition 5.16, we have
\[
\Pr\Bigl(\frac1n\sum_{i=1}^n(x_i^T\Sigma_{\cdot,j})^2 - \Sigma_{j,j}^{-1} > t\Bigr) \le 2e^{-c_1\min\{\frac{nt^2}{16\sigma_x^2(\Sigma_{j,j}^{-1})^2},\frac{nt}{4\sigma_x\Sigma_{j,j}^{-1}}\}}
\]
for some absolute constant $c_1 > 0.$ For $t = \Sigma_{j,j}^{-1},$ the bound simplifies to
\[
\Pr\Bigl(\frac1n\sum_{i=1}^n(x_i^T\Sigma_{\cdot,j})^2 - \Sigma_{j,j}^{-1} > \Sigma_{j,j}^{-1}\Bigr) \le 2e^{-c_1\min\{\frac{n}{16\sigma_x^2},\frac{n}{4\sigma_x}\}}.
\]
We take a union bound over $j\in[p]$ to obtain the stated result.
\end{proof}

\begin{proof}[Proof of Lemma \ref{lem:c-bias-bounded}]
We follow a similar argument as the proof of Lemma \ref{lem:c-var-bounded}:
\[
\hSigma_{k;j,j} = \hSigma_{j,j} = \hSigma_{j,j} - \Sigma_{j,j} + \Sigma_{j,j} = \frac1n\sum_{i=1}^nx_{i,j}^2 - \Sigma_{j,j} + \Sigma_{j,j}.
\]
Since the $z_i = \Sigma^{-\frac12}x_i$ is subgaussian with subgaussian norm $\sigma_x,$ $x_{i,j}$ is also subgaussian with subgaussian norm bounded by
\[
\|x_{i,j}\|_{\psi_2} \le \|\Sigma_{j,\cdot}^{\frac12}z_i\|_{\psi_2}\le \sigma_x(\Sigma_{j,j})^{\frac12}.
\]
We recognize $\hSigma_{j,j} - \Sigma_{j,j} = \frac1n\sum_{i=1}^n x_{i,j}^2 - \Sigma_{j,j}$ as a sum of \iid{} subexponential random variables with subexponential norm bounded by
\[
\|\hSigma_{j,j} - \Sigma_{j,j}\|_{\psi_1} \le 2\|x_{i,j}^2\|_{\psi_1} \le 4\|x_{i,j}\|_{\psi_2}^2 \le 4\sigma_x^2\Sigma_{j,j}.
\]
By \cite{vershynin2010introduction}, Proposition 5.16, we have
\[
\Pr(\hSigma_{j,j} - \Sigma_{j,j} > t) \le 2e^{-c_1\min\{\frac{nt^2}{16\sigma_x^2\Sigma_{j,j}^2},\frac{nt}{\sigma_x\Sigma_{j,j}}\}}
\]
for some absolute constant $c_1 > 0.$ For $t = \Sigma_{j,j},$ the bound simplifies to
\[
\Pr(\hSigma_{j,j} - \Sigma_{j,j} > \Sigma_{j,j}) \le 2e^{-c_1\min\{\frac{n}{16\sigma_x^2},\frac{n}{4\sigma_x}\}}.
\]
We take a union bound over $j\in[p]$ to obtain the stated result.
\end{proof}

\section{A sharper consistency result}
\label{sec:tbetaht-consistency-2-pf}

It is possible to obtain a sharper consistency result by forgoing the $\ell_\infty$ norm convergence rate. By sharper, we mean the sample complexity of the averaged estimator from $m \lesssim \frac{n}{s_0^2\log p}$ to $m \lesssim \frac{n}{s_0\log p}$.

\begin{theorem}
\label{thm:tbetaht-consistency-2}
Under the conditions of Theorem \ref{thm:parallel-avg-estimator-consistency-1}, hard-thresholding $\tbeta$ at $t = |\tbeta|_{(\hs_0)}$ for some $\hs_0 \sim s_0$, \ie{} setting all but the largest $s_0'$ debiased coefficients to zero, gives
\BNUM
\item $\|\tbetaht - \betas\|_2 \lesssim_P \bigl(\frac{s_0\log p }{N}\bigr)^{\frac12} + \frac{s_0\log p}{n}$,
\item $\|\bbetaht -\betas\|_1 \lesssim_P \bigl(\frac{s_0^2\log p }{N}\bigr)^{\frac12} + \frac{s_0^{3/2}\log p}{n}$.
\ENUM
\end{theorem}

The sharper consistency result depends on a result by \cite{javanmard2013nearly}, which we combine with Lemma \ref{lem:vdG-LDPE-consistency} and restate for completeness. Before stating the results, we define the $(\infty,l)$ norm of a point $x\in\reals^p$ as
\[\textstyle
\norm{x}_{(\infty,l)} := \max_{\cA\subset[p],|\cA| \ge l}\frac{\norm{x_{\cA}}_2}{\sqrt{l}}.
\]
When $l = 1$, the $(\infty,l)$ norm of $x$ is its $\ell_\infty$ norm. When $l = p$, the $(\infty,l)$ norm is the $\ell_2$ norm (rescaled by $\frac{1}{\sqrt{p}}$). Thus the $(\infty,l)$ norm interpolates between the $\ell_2$ and $\ell_\infty$ norms. \cite{javanmard2013nearly}, Theorem 2.3 shows that the bias of the debiased lasso is of order $\frac{\sqrt{s_0}\log p}{n}$.

\begin{lemma}
\label{lem:vdG-LDPE-bias-small}
Under the conditions of Theorem \ref{thm:parallel-avg-estimator-consistency-1},
\[
\|\hDelta_k\|_{(\infty,c's_0)} \lesssim_P \frac{c\sqrt{s_0}\log p}{n}\text{ for any }k\in[m]\text{ for any }c' > 0,
\]
where $c$ is a constant that depends only on $c'$ and $\Sigma$.
\end{lemma}

\begin{proof}
The result is essentially \cite{javanmard2013nearly}, Theorem 2.3 with $\hat{\Omega} = \hTheta$ given by \eqref{eq:vdG-LDPE}. Lemma \ref{lem:vdG-LDPE-consistency} shows that
\[\textstyle
\max_{j\in[p]}\|\hTheta_{j,\cdot} - \Sigma^{-1}_j\|_1\lesssim_P s_j\bigl(\frac{\log p}{n}\bigr)^\frac12,
\]
Since $\frac{\max_{j\in[p]}s_j^2\log p}{n} \sim o(1)$, $\hTheta$ satisfies the conditions of \cite{javanmard2013nearly}, Theorem 2.3:
\[
\|\hDelta_k\|_{(\infty,c's_0)} \lesssim_P \frac{c\sqrt{s_0}\log p}{n}\text{ for any }k\in[m],
\]
The bound is uniform in $k\in[m]$ by a union bound for suitable parameters $\lambda_k\sim\bigl(\frac{\log p}{n}\bigr)^{\frac12}$.
\end{proof}

By Lemma \ref{lem:vdG-LDPE-bias-small}, the estimator \eqref{eq:parallel-avg-estimator} is consistent in the $(\infty,s_0)$ norm. The argument is similar to the proof of Theorem \ref{thm:parallel-avg-estimator-consistency-1}.

\begin{theorem}
\label{thm:parallel-avg-estimator-consistency-2}
Under the conditions of Theorem \ref{thm:parallel-avg-estimator-consistency-1},
\[
\|\bbeta - \betas\|_{(\infty,c's_0)} \sim O_P\Bigl(\Bigl(\frac{\log p }{N}\Bigr)^{\frac12} + \frac{\sqrt{s_0}\log p}{n}\Bigr).
\]
\end{theorem}

\begin{proof}
We start by substituting the linear model into \eqref{eq:parallel-avg-estimator}:
\[
\tbeta = \frac1m\sum_{k=1}^m \hDelta_k + \frac{1}{N}\hTheta X^T\epsilon.
\]
Subtracting $\betas$ and taking norms, we obtain
\BEQ
\|\tbeta - \betas\|_{(\infty,c's_0)} \le \frac1m\sum_{k=1}^m \|\hDelta_k\|_{(\infty,c's_0)} + \bigl\|\frac{1}{N}\hTheta X^T\epsilon\bigr\|_{(\infty,c's_0)}.
\label{eq:parallel-avg-estimator-consistency-2-1}
\EEQ
By Lemma \ref{lem:vdG-LDPE-bias-small}, the first (bias) term is of order $\frac{c\sqrt{s_0}\log p}{n}$. We focus on showing the second (variance) term is of order $\bigl(\frac{\log p}{N}\bigr)^{\frac12}$. Since the $(\infty,l)$ norm is non-increasing in $l$,
\[
\bigl\|\frac{1}{N}\hTheta X^T\epsilon\bigr\|_{(\infty,c's_0)} \le \bigl\|\frac{1}{N}\hTheta X^T\epsilon\bigr\|_\infty.
\]
By \cite{vershynin2010introduction}, Proposition 5.16 and Lemma \ref{lem:c-var-bounded}, it is possible to show that
\[
\bigl\|\frac{1}{N}\hTheta X^T\epsilon\bigr\|_\infty \sim O_P\Bigl(\Bigl(\frac{\log p}{N}\Bigr)^{\frac12}\Bigr).
\]
Thus the second term in \eqref{eq:parallel-avg-estimator-consistency-2-1} is of order $\bigl(\frac{\log p}{N}\bigr)^{\frac12}$. We put all the pieces together to obtain the stated conclusion.
\end{proof}

We are ready to prove Theorem \ref{thm:tbetaht-consistency-2}. Since $\tbetaht - \betas$ is $2s_0$-sparse,
\[
\|\tbetaht - \betas\|_2^2 \lesssim s_0\|\tbetaht - \betas\|_{(\infty,c's_0)}^2
\]
or, equivalently,
\[
\|\tbetaht - \betas\|_2 \lesssim \sqrt{s_0}\|\tbetaht - \betas\|_{(\infty,c's_0)}.
\]
By the triangle inequality,
\begin{align*}
\|\tbetaht - \betas\|_{(\infty,c's_0)} &\le \|\tbetaht - \tbeta\|_{(\infty,c's_0)} + \|\tbeta - \betas\|_{(\infty,c's_0)} \\
&\le 2\|\tbeta - \betas\|_{(\infty,c's_0)},
\end{align*}
where the second inequality is by the fact that thresholding at $t = |\tbeta|_{(c's_0)}$ minimizes $\|\beta - \betas\|_{(\infty,c's_0)}$ over $c's_0$-sparse points $\beta$.  Thus
\[
\|\tbetaht - \betas\|_2 \sim O_P\Bigl(\Bigl(\frac{s_0\log p }{N}\Bigr)^{\frac12} + \frac{s_0\log p}{n}\Bigr).
\]
To complete the proof of Theorem \ref{thm:tbetaht-consistency-2}, we observe that the consistency of $\tbetaht$ in the $\ell_1$ norm follows by the fact that $\tbetaht - \betas$ is $2s_0$-sparse.

By Theorem \ref{thm:tbetaht-consistency-2}, when $m \lesssim \frac{N}{s_0 \log p},$ the variance term is dominant and the convergence rates given by the theorem simplify to the convergence rates of the (centralized) lasso estimator:
\BNUM
\item $\|\bbetaht- \betas\|_2 \lesssim_P \bigl(\frac{s_0\log p}{N}\bigr)^{\frac12},$
\item $\|\bbetaht- \betas\|_1 \lesssim_P \bigl(\frac{s_0^2\log p}{N}\bigr)^{\frac12}.$
\ENUM
Thus, by forgoing consistency in the $\ell_\infty$ norm, it is possible to reduce the sample complexity of the averaged estimator to $m \lesssim \frac{s_0\log p}{N}$. When $m = 1,$ we recover the sample complexity of the centralized lasso estimator.

Theorem \ref{thm:tbetaht-consistency-2} requires an estimate of $s_0$. To wrap up, we mention that it is possible to obtain a good estimate of $s_0$ by the \emph{empirical sparsity} of any of the local lasso estimators. Let $\hE\subset[p]$ be the \emph{equicorrlation set} of the lasso estimator.
\[
\{j\in[p]:|x_j^T(y - X\hbeta)| = \lambda\}.
\]
The empirical sparsity $\hs_0$ is the size of $\hE$.

\begin{lemma}[\cite{sun2015regularization}, Lemma 6.20]
\label{lem:hs-bound}
Under (A1)--(A3), when
\[\textstyle
n > \max\{4000 \ts_0 \sigma_x^2 \log(\frac{60\sqrt{2}ep}{\ts_0}),\,4000 \sigma_x^4 \log p,\,s_0\log p\},
\]
where $\ts_0 := s_0 + 25920\kappa s_0$, we have
\[
\hs_0 \le \Bigl(\frac{192\sigma_x^2 + 384\lambda_{\max}(\Sigma)}{\lambda_{\min}(\Sigma)} + \frac{384^2\sigma_x^4}{c_1\lambda_{\min}(\Sigma)^2}\Bigr)^2s
\]
with probability at least $1 - 2p^{-(s_0+1)}$.
\end{lemma}

\bibliographystyle{imsart-nameyear}
\bibliography{yuekai,averaging}

\end{document}